\documentclass{article}

\usepackage{microtype}
\usepackage{graphicx}
\usepackage{subcaption}
\usepackage{booktabs} % for professional tables

\usepackage{hyperref}
\hypersetup{
 colorlinks=true,
 linkcolor=blue,
 urlcolor=blue,
 citecolor=blue
}
\usepackage{graphicx} % DO NOT CHANGE THIS
\usepackage{booktabs}
\usepackage{subcaption} 
\usepackage{caption} % DO NOT CHANGE THIS AND DO NOT ADD ANY OPTIONS TO IT
\usepackage{multirow} 
\usepackage{float}
\usepackage{todonotes}
\usepackage{pgfplots}
\usepackage[normalem]{ulem} 
\pgfplotsset{compat=1.18}  % Use one compatibility setting for PGFPlots
\usepackage{tikz}
\usepgfplotslibrary{groupplots}

\usepackage[table]{xcolor}
\usepackage{xcolor}
\usepackage{mathtools}           % Loads amsmath, plus extras
\usepackage{amssymb}             % Blackboard symbols
\usepackage{amsthm}              % Theorem environments

\usepackage{algorithm}
\usepackage{algorithmic}

\floatstyle{ruled}
\newfloat{algorithm}{tb}{alg}
\floatname{algorithm}{Algorithm}

\usepackage{listings}
\theoremstyle{plain}
\newtheorem{theorem}{Theorem}[section]
\newtheorem{proposition}[theorem]{Proposition}

\theoremstyle{definition}
\newtheorem{definition}[theorem]{Definition}

\theoremstyle{remark}

\definecolor{group1}{RGB}{245,250,255}  % Lighter Alice Blue
\definecolor{group2}{RGB}{250,255,250}  % Lighter Mint Cream
\definecolor{group3}{RGB}{255,255,245}  % Lighter Papaya Whip
\definecolor{group4}{RGB}{255,245,255}  % Lighter Lavender Blush
\definecolor{group5}{RGB}{255,240,235}  % Lighter Light Pink
\usepackage{natbib}

\usepackage{hyperref}

\usepackage[preprint]{icml2026}

\usepackage{amsmath}
\usepackage{amssymb}
\usepackage{mathtools}
\usepackage{amsthm}

\usepackage[capitalize,noabbrev]{cleveref}
\icmltitlerunning{Submission and Formatting Instructions for ICML 2026}

\begin{document}

\twocolumn[
  \icmltitle{Revisiting Reweighted Risk for Calibration: AURC, Focal and Inverse Focal Loss}



  \icmlsetsymbol{equal}{*}
% \aistatsauthor{ Han Zhou \And
% Sebastian G.~Gruber\And
% Teodora Popordanoska\And
% Matthew B. Blaschko }
% \aistatsaddress{ ESAT-PSI, KU Leuven, Belgium }

  \begin{icmlauthorlist}
    \icmlauthor{Han Zhou}{yyy}
     \icmlauthor{Sebastian G.~Gruber}{yyy}
    \icmlauthor{Teodora Popordanoska}{yyy}
     \icmlauthor{Matthew B. Blaschko}{yyy}
  \end{icmlauthorlist}

  \icmlaffiliation{yyy}{ESAT-PSI, KU Leuven, Belgium }

  \icmlcorrespondingauthor{Han Zhou}{han.zhou@esat.kuleuven.be}

  \icmlkeywords{Machine Learning, ICML}

  \vskip 0.3in
]
\printAffiliationsAndNotice{}
\begin{abstract}
Several variants of reweighted risk functionals, such as focal loss, inverse focal loss, and the Area Under the Risk–Coverage Curve (AURC), have been proposed for improving model calibration; yet their theoretical connections to calibration errors remain under-explored. In this paper, we revisit a broad class of weighted risk functions and find a principled connection between calibration error and selective classification. We show that minimizing calibration error is closely linked to the selective classification paradigm and demonstrate that optimizing selective risk in low-confidence regions naturally improves calibration. Our proposed loss shares a similar reweighting strategy with dual focal loss but offers greater flexibility through the choice of confidence score functions (CSFs). Furthermore, our approach utilizes a bin-based cumulative distribution function (CDF) approximation, enabling efficient gradient-based optimization with $\mathcal{O}(nM)$ complexity for $n$ samples and $M$ bins. Empirical evaluations demonstrate that our method achieves competitive calibration performance across a range of datasets and model architectures.
\end{abstract}

\section{Introduction}
Over the past decade, advances in deep learning have enabled neural networks to achieve, and in some cases surpass, human-level performance on tasks ranging from image recognition~\citep{haggenmuller2021skin} to natural language understanding~\citep{achiam2023gpt, liu2024deepseek}. However, achieving high accuracy does not necessarily ensure that the predicted confidence scores align with the true probabilities of outcomes, a phenomenon known as miscalibration~\citep{guo2017calibration}. Miscalibration can lead to critical errors in high‐stakes applications where trustworthy uncertainty estimates are essential. To address and quantify this gap, the research community has introduced a range of calibration error metrics~\citep{brier1950verification, guo2017calibration, kumar2019verified, popordanoska2022consistent} to evaluate the discrepancy between predicted confidences and observed outcome frequencies. Recent work has explored both trainable~\citep{karandikar2021soft, kumar2018trainable, krishnan2020improving, bohdal2023meta, popordanoska2022consistent} and post-hoc~\citep{guo2017calibration, ding2021local, kull2019beyond, wenger2020non, Ma2021a} approaches for neural network calibration, achieving impressive reductions in terms of those empirical calibration errors. Among trainable methods, focal loss~\citep{lin2017focal}—which emphasizes hard examples during training—has gained popularity for its calibration benefits. Conversely, inverse focal loss~\citep{wang2021rethinking}, which focuses on the easy examples more, has also been shown to improve calibration in certain settings. This raises a key question: \emph{Which weighting scheme is most effective for improving model calibration?} 

We address this question by examining the relationship between selective classification \citep{geifman2017selective} and calibration error. We revisit the exisiting theoretical results that links calibration error minimization to the principles of selective classification, showing that optimal calibration strategies naturally emerge from selective weighting based on confidence rankings. Building on this connection, we propose a loss function similar to the AURC \citep{geifman2018bias}, but focused specifically on selective risk in low-confidence regions that require careful calibration. Unlike AURC, our method employs a bin-based CDF approximation that enables efficient gradient-based optimization with $\mathcal{O}(nM)$ complexity. This adaptation ensures computational scalability while maintaining differentiability with respect to any choice of CSFs.

Our main \textbf{contributions} are summarized as follows:
% \vspace{-3mm}
\begin{itemize}
    \setlength\parskip{0pt}   % ensures no extra paragraph spacing
    \item We revisit the calibration error minimization and its relation to the “abstain” strategy at the core of selective classification.
    \item Motivated by this connection, we propose a computationally efficient, differentiable loss function based on selective classification principles, using a bin-based CDF approximation that avoids expensive sorting operations while maintaining gradient flow for any CSF choice.
    \item Through extensive experiments, we show that our proposed method delivers competitive calibration performance across various datasets and network architectures, surpassing the majority of existing calibration-aware loss functions.
\end{itemize}

\section{Related work}
\textbf{Post-hoc calibration} methods address miscalibration by learning a mapping function that transforms model outputs into calibrated posterior probabilities. Approaches like Platt scaling~\citep{platt1999probabilistic}, temperature scaling (TS)~\citep{guo2017calibration}, and local temperature scaling~\citep{ding2021local} typically preserve classification accuracy by introducing a scalar temperature parameter. In contrast, other probabilistic methods—including Beta calibration~\citep{kull2017beta}, Dirichlet calibration~\citep{kull2019beyond}, Gaussian processes~\citep{wenger2020non}, and the Concrete distribution~\citep{esaki2024accuracy}—may achieve better calibration but often at the cost of reduced classification accuracy. Notably, Dirichlet calibration generalizes Beta calibration to multi-class settings. \cite{Ma2021a} and \cite{liu2024optimizing} proposed post-hoc calibration methods that adjust predicted probabilities based on correctness, using binary classifiers to either distinguish between correct and incorrect predictions or to directly modulate confidence accordingly. However, these approaches often involve a trade-off between calibration and predictive accuracy. We aim to develop a trainable calibration method that improves calibration performance without significantly compromising predictive accuracy.

\noindent \textbf{Trainable calibration methods} seek to enhance the calibration of deep neural networks by explicitly modifying the training objective. One class of approaches introduces differentiable surrogate losses that directly approximate the expected calibration error, as demonstrated in~\citep{karandikar2021soft, kumar2018trainable, krishnan2020improving, bohdal2023meta, popordanoska2022consistent}. \citet{popordanoska2022consistent} proposed $\mathrm{ECE}^{\mathrm{KDE}}$, a differentiable estimator of calibration error, and introduced the KDE-XE objective, which augments cross-entropy (XE) with an ECE-KDE regularizer to balance accuracy and calibration. Alternatively, several studies replace the standard cross-entropy loss with calibration-aware objectives, including mean squared error loss~\citep{hui2020evaluation}, inverse focal loss~\citep{wang2021rethinking}, and focal loss~\citep{lin2017focal}. Focal loss, in particular, has been empirically shown to improve calibration by mitigating overconfident predictions~\citep{mukhoti2020calibrating, charoenphakdee2021focal, komisarenko2024improving}. Nevertheless, \citet{wang2021rethinking} argue that the regularization effect of focal loss might suppress crucial information about sample difficulty, which in turn limits the effectiveness of subsequent post-hoc calibration techniques. To address this trade-off, several variants of focal loss have been proposed, such as Adaptive Focal Loss~\citep{ghosh2022adafocal} and Dual Focal Loss~\citep{tao2023dual}. Distinct from these prior approaches, we propose a loss function based on the selective classification framework. While selective classification is not originally intended for calibration, we show that its underlying principle is fundamentally aligned with minimizing calibration error.

\noindent\textbf{Selective classification} augments a standard classifier with a reject option, enabling an explicit trade-off between selective risk and coverage. This technique mitigates selective risk by withholding low-confidence predictions~\citep{geifman2017selective, geifman2018bias, ding2020revisiting, galilcan}. The Area Under the Risk-Coverage Curve (AURC) and its normalized variant, Excess-AURC (E-AURC)~\citep{geifman2018bias}, are among the most widely used evaluation metrics for selective classification systems. These metrics compute the risk associated with accepted predictions across different confidence thresholds. \citet{zhou2024novelcharacterizationpopulationarea} showed that AURC can be interpreted as a weighted risk function, where the weights are determined by the ranking induced by a CSF. Our work is motivated by the goal of establishing a connection between selective classification and model calibration. We propose a differentiable loss function that emphasize the selective risks in low-confidence region where calibration is most critical. Unlike AURC computation that requires expensive sorting operations, our approach uses a bin-based CDF approximation to achieve computational efficiency while maintaining differentiability, making it practical for large-scale deep learning applications.

\section{Preliminaries}
A common approach to training a classifier with a softmax output is to learn a function \( f : \mathcal{X} \to \Delta^{k-1}\), which maps an input \( \boldsymbol{x} \in \mathbb{R}^d \) to a probability distribution over \( k \) classes. The output \( \boldsymbol{p} = f(\boldsymbol{x}) = (p_{1}, \dots, p_{k})^{\top} \in \Delta^{k-1} \) represents the predicted class probabilities. Given an input \( \boldsymbol{x} \), let \( y' \in \{1, \dots, k\} \) denote the ground-truth label, and define $\boldsymbol{y} \in \{0,1\}^k$ its one-hot representation.

\subsection{Weighted Risk Functions} 
To train the classifier \( f \), one typically minimizes the expected loss over the joint probability distribution \( \mathbb{P}(\boldsymbol{x}, \boldsymbol{y}) \):
\begin{equation}
    \mathcal{R}(f) = \mathbb{E}[\ell(f \left( \boldsymbol{x} \right), \boldsymbol{y})],
\end{equation}
where \( \ell \) is a loss function, such as cross-entropy or squared loss. While conventional losses treat all samples uniformly, this may be suboptimal under data imbalance or varying prediction confidence. A widely adopted alternative is the focal loss, a weighted loss function that emphasizes uncertain or misclassified examples.

\begin{definition}[Focal Loss \citep{lin2017focal}]
The Focal Loss (FL) is defined as a reweighted cross-entropy that emphasizes hard, misclassified examples:
\begin{equation}
\ell_{\text{FL}}(\boldsymbol{p}, \boldsymbol{y}) = -(1 - p_{y'})^\gamma \ln(p_{y'}), \quad \gamma \ge 0,
\end{equation}
where $p_{y'}$ is the predicted probability for the ground-truth class $y'$. 
\end{definition}
FL recovers standard cross-entropy at $\gamma = 0$. For $\gamma > 0$, the loss acts as a reweighted objective where each instance is modulated by $w(p_{y'}) = (1 - p_{y'})^\gamma$. The focusing parameter $\gamma$ controls the importance distribution: as $p_{y'} \to 1$, high-confidence samples are significantly downweighted, shifting the optimization focus toward low-confidence examples.

\begin{definition}[Inverse Focal Loss \citep{wang2021rethinking}]
The Inverse Focal Loss (IFL) reverses the weighting logic of FL by emphasizing high-confidence predictions:
\begin{align}
 \ell_{\text{IFL}}(\boldsymbol{p}, \boldsymbol{y}) = - (1 + p_{y'})^{\gamma} \ln(p_{y'}), \quad \gamma \geq 0,
\end{align}
where $p_{y'}$ is the predicted probability for the ground-truth class $y'$. 
\end{definition}
Unlike focal loss, inverse focal loss prioritizes high-confidence samples. Our gradient analysis reveals that empirically that certain values of \(\gamma\) may lead to undesirable behavior during training (see Fig.~\ref{fig:gammaselection} in Appendix~\ref{appendix:gradient}). Surprisingly, despite their fundamentally opposite reweighting schemes, both strategies have been reported to improve classifier calibration.

\begin{figure}
\centering
\begin{minipage}[t]{0.49\linewidth}
    \centering
    \includegraphics[width=1.6in]{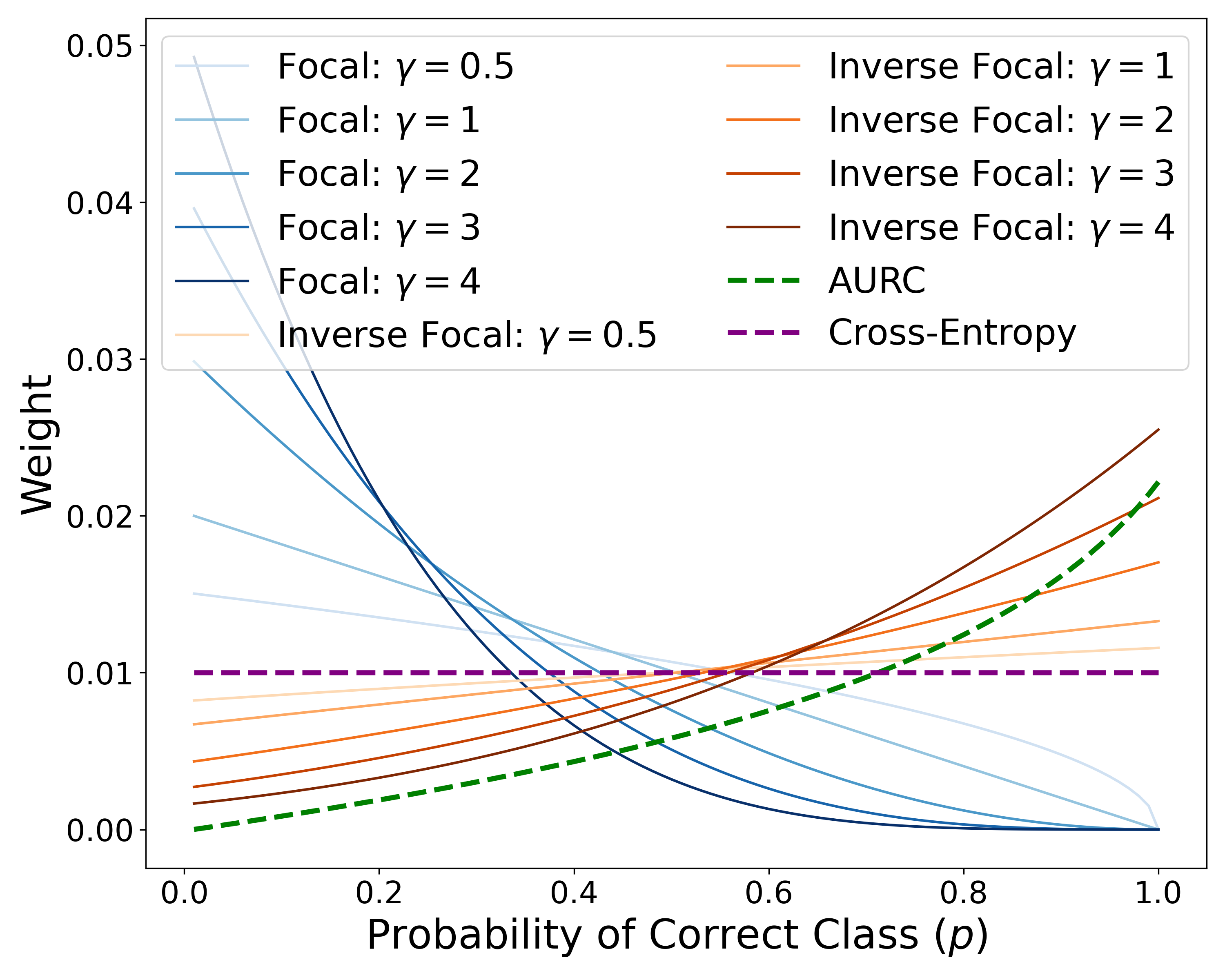}
    \subcaption{Normalized Weights}
    \label{fig:weights_comparison}
\end{minipage}%
\begin{minipage}[t]{0.49\linewidth}
    \centering
    \includegraphics[width=1.6in]{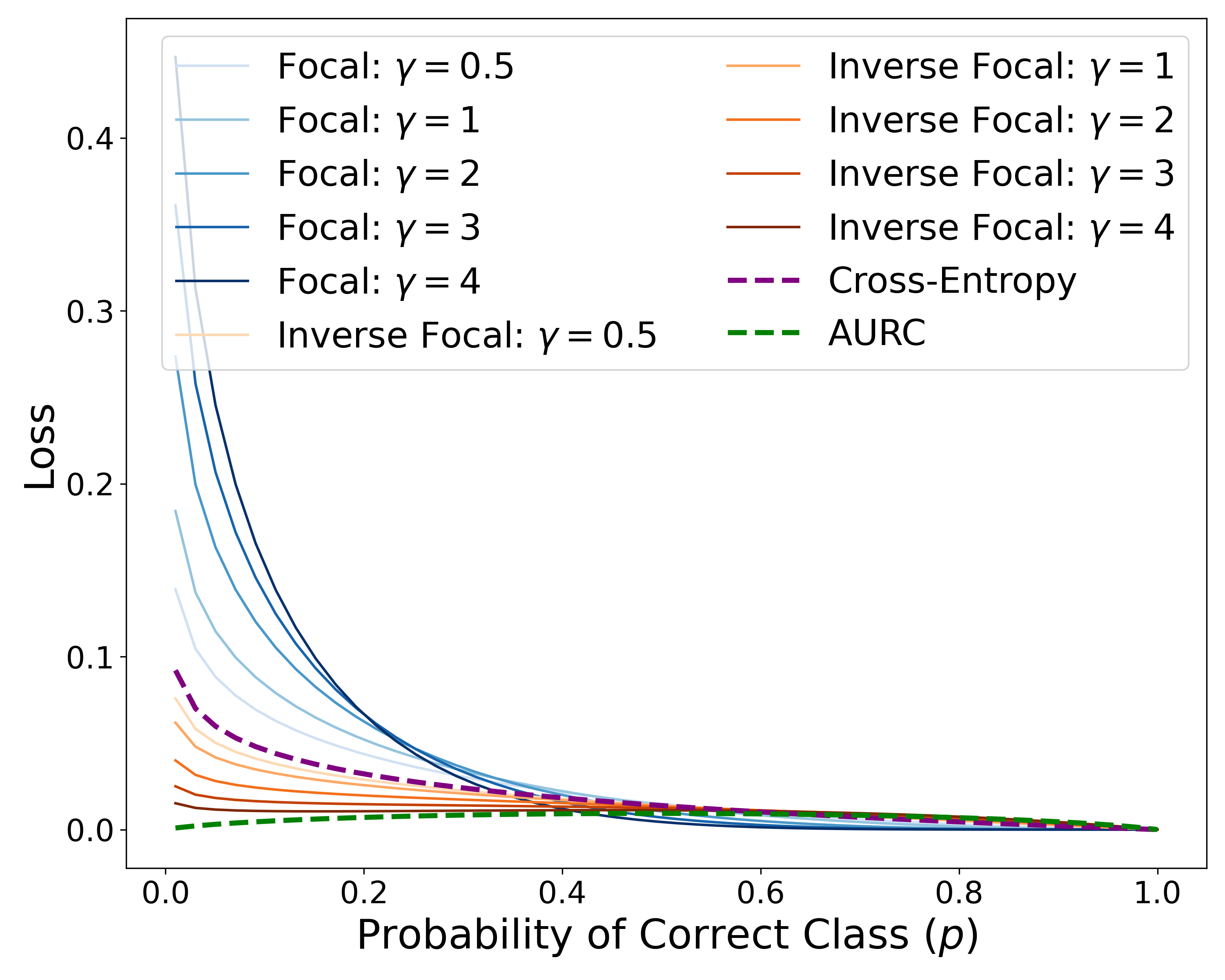}
    \subcaption{Loss } %
    \label{fig:loss_comparison}
\end{minipage}
\caption{Comparison of weight and loss (with normalized weights) behavior with respect to prediction probability \( p \). We normalize the weights in both FL and inverse focal loss for fair comparison.}
\label{fig:weightandloss}
\end{figure}

\subsection{Selective classification}
Selective classification pairs a classifier \( f \) with a \textit{confidence scoring function} (CSF) \( g : \mathcal{X} \rightarrow \mathbb{R} \) that quantifies the reliability of individual predictions. The resulting selective classifier, denoted as the pair \( (f, g) \), only outputs a prediction if the confidence score meets a specified threshold \( \tau \); otherwise, it abstains:
\begin{equation}
(f, g)(\boldsymbol{x}; \tau) := 
\begin{cases} 
f(\boldsymbol{x}) & \text{if } g(\boldsymbol{x}) \ge \tau, \\
\text{"abstain"} & \text{if } g(\boldsymbol{x}) < \tau.
\end{cases}
\label{eq:selectiveclassifier}
\end{equation}
The performance of such a classifier is evaluated via the \textit{selective risk} and \textit{coverage}. For a loss function $\ell$, the selective risk is defined as:
\begin{equation} \label{eq:selectiverisk}
   \mathcal{R}(f, g; \tau) := \frac{\mathbb{E}_{(\boldsymbol{x}, \boldsymbol{y}) \sim \mathbb{P}(\boldsymbol{x}, \boldsymbol{y})} \bigl[\ell(f(\boldsymbol{x}),\boldsymbol{y}) \cdot \mathbb{I}\left(g(\boldsymbol{x}) \geq \tau \right)\bigr]}{\Pr\left(g(\boldsymbol{x}) \geq \tau \right)},
\end{equation}
where the denominator denotes the \textit{coverage}, representing the fraction of samples for which the model provides a prediction. 

The threshold \( \tau \) governs the inherent trade-off between risk and coverage. Rather than evaluating at a fixed operating point, the Area Under the Risk-Coverage Curve (AURC) \citep{geifman2018bias} provides a holistic measure of ranking robustness by aggregating selective risk across all possible thresholds. As established in \citet{zhou2024novelcharacterizationpopulationarea}, this metric can be reformulated as a weighted risk objective, which we detail in Definition~\ref{def:aurc}.
\begin{definition}[AURC] \label{def:aurc} Let $S = g(\boldsymbol{X})$ be the random variable of confidence scores. The cumulative distribution function (CDF) of $S$ is defined as:
\begin{equation} \label{eq:cumulativeCSF}
G(s) = \Pr(S \le s) = \mathbb{E}_{\boldsymbol{x}' \sim \mathbb{P}_x} [\mathbb{I}(g(\boldsymbol{x}') \le s)].
\end{equation} 
The AURC metric can be equivalently expressed as:
\begin{equation} \label{eq:aurc}
    \text{AURC}(f) = -\mathbb{E} \Bigl[\ln \bigl(1 - G(s) \bigr) \cdot \ell(f(\boldsymbol{x}), \boldsymbol{y}) \bigr].
\end{equation}
\end{definition}
This formulation shows that AURC can be interpreted as a reweighted risk, utilizing a weighting scheme similar to that of the inverse focal loss (see Fig.~\ref{fig:weightandloss}). In this scheme, each sample is weighted by a function that depends on the CSF. Common choices for the CSFs include Maximum Softmax Probability (MSP)~\citep{hendrycks2016baseline}, Negative Entropy~\citep{liu2020energy}, and Softmax Margin~\citep{belghazi2021classifiers} as described in Appendix~Table~\ref{tb:csfs}.

\subsection{Calibration Errors}
Below, we provide the definitions of commonly used calibration errors and their empirical estimators.
\begin{definition}[Perfect Calibration]
A model \( f: \mathcal{X} \rightarrow \Delta^{k-1} \) is \emph{calibrated} if and only if \( f(\boldsymbol{x}) = \mathbb{E}[\boldsymbol{y} \mid f(\boldsymbol{x})] \) holds almost surely.
\end{definition}

\begin{definition}[\(L_\rho\) Calibration Error (\( \text{CE}_\rho \))~\citep{kumar2019verified}]
The \( L_\rho \) calibration error of \( f \), for \( \rho \ge 1 \), is defined as:
\begin{align}
\text{CE}_\rho(f) = \left( \mathbb{E} \left[ \| f(\boldsymbol{x}) - \mathbb{E}[\boldsymbol{y} \mid f(\boldsymbol{x})] \|_\rho^\rho \right] \right)^{1/\rho}, 
\end{align}
where the expectation is taken with respect to the joint distribution \( \mathbb{P}(\boldsymbol{x}, \boldsymbol{y}) \).
\end{definition}
In practice, the true value of \( \text{CE}_\rho(\cdot) \) is unobservable, as it depends on the unknown conditional distribution \( \mathbb{P}( \boldsymbol{y} \mid f(\boldsymbol{x})) \). In the multi-class setting, this conditional distribution becomes even more complex. Consequently, more research has shifted toward more tractable metrics such as top-label and class-wise calibration errors.

\begin{definition}[Top-label Calibration Error (\( \text{ECE} \)~\citep{naeini2015obtaining})]
The top-label \( \text{ECE} \) of a model \( f: \mathcal{X} \to \Delta^{k-1} \) is defined as:
\begin{align}
\text{ECE}(f) =  \mathbb{E} \bigl[ | f_{\hat{y}}(\boldsymbol{x}) - \mathbb{E}[y= \hat{y} \mid f_{\hat{y}}(\boldsymbol{x}] | \bigr]
\end{align}
where the expectation is taken with respect to \( \mathbb{P}(\boldsymbol{x}, \boldsymbol{y}) \), and \( \hat{y} = \arg\max_{c\in [k]} f_c(\boldsymbol{x}) \) is the predicted top-label for the input \( \boldsymbol{x} \).
\end{definition}
This is referred to as the \emph{Expected Calibration Error }(ECE)~\citep{guo2017calibration}. For this quantity, \citet{naeini2015obtaining} proposed an empirical estimator that uses a binning scheme to estimate it, denoted by \(\widehat{\text{ECE}}_B\).

\begin{definition}[Binned estimator $\widehat{\text{ECE}}_{B}$~\citep{naeini2015obtaining}]
Let the binning scheme be
$B=\{I_{b} \big|b=1,\dots,m\}$, the binned estimator is denoted as:
   \begin{align}
         \widehat{\text{ECE}}_{B}(f) = \sum_{b=1}^{m} \frac{n_b}{n} \left| \text{acc}_b (f) - \text{conf}_b(f) \right|
   \end{align}
   where $n_b$ is the number of samples falling into bin $I_{b}$, $n$ is the total number of samples, $\text{conf}_b(f)$ is the average predicted confidence in bin $ I_{b} $, and $\text{acc}_b(f)$ is the accuracy in that bin.
\end{definition}
Beyond the top prediction, the \emph{class-wise calibration error} (cwECE) extends this notion to all category-specific margins. 
\begin{definition}[Class-Wise Calibration Error (\( \text{cwECE}\)~\citep{panchenko2022class})]
The cwECE is the arithmetic mean of the marginal calibration errors:
\begin{align}
\text{cwECE}(f) = \frac{1}{k} \sum_{c=1}^{k} \mathbb{E} \bigl[ \left| f_c(\boldsymbol{x}) - \Pr(y = c \mid f_c(\boldsymbol{x})) \right| \bigr]
\end{align}
where \( f_c(\boldsymbol{x}) \) is the predicted probability for class \( c \).
\end{definition}
We follow related works~\citep{gupta2022top} in defining this metric as the average calibration error across classes. Similar to ECE, it can be estimated empirically using a binning scheme \( B \), resulting in a bin-based estimator \( \widehat{\text{cwECE}}_B(f) \).
\begin{definition}[Binned estimator $\widehat{\text{cwECE}}_{B}$] \label{def:bincwece}
Let $m_c$ be the number of bins for class $c$ and the binning scheme be
$B=\{I_{c,b} \big|\ c=1,\dots,k; b=1,\dots,m_c\}$, The binned estimator for \( \rho = 1 \) is given by:
\begin{align}
\widehat{\mathrm{cwECE}}_{B}(f)
& =\frac{1}{k} \sum_{c=1}^{k}\frac1n\sum_{b=1}^{m_c}
n_{c,b}\Bigl|\text{acc}_{c,b}
-\text{conf}_{c,b}\Bigr|  \label{eq:bincwece}
\end{align}
where 
\begin{align}
    \text{acc}_{c,b} = \frac{1}{n_{c,b}}\!\!\sum_{i\in I_{c,b}}\!\!\mathbb{I}(y_i^\prime =c), \quad   \text{conf}_{c,b} =  \frac{1}{n_{c,b}}\!\!\sum_{i\in I_{c,b}}\!f_c(\boldsymbol{x}_i) , \nonumber
\end{align}
\( n_{c,b} \) is the number of samples in bin \( I_{c,b} \), and \( n \) is the total number of samples.
\end{definition}

\section{A Differentiable Loss Function Derived from Selective Classification}
In this section, we begin by presenting some existing results related to calibration error—as presented in Prop.~\ref{prop:ece-lb}. These results elucidate how calibration error can be reduced for a pre-trained model and motivate a deeper exploration of the connection between selective classification and calibration. Building on these insights, we propose a differentiable loss function derived from selective classification. The formulation is fully differentiable—enabled by a bin-based CDF approximation and supports arbitrary CSFs.
\subsection{Linking Calibration Error Minimization to Selective Classification}
\begin{proposition}[Binning underestimates ECE {\citep{Ma2021a, kumar2019verified}}]\label{prop:ece-lb}
For a classifier $f:\mathcal X \to \Delta^{k-1}$, the top‑label ECE satisfies the following lower bound over all possible binning schemes $B \in \mathcal{B}$:
\begin{equation}
  \operatorname{ECE}(f) \ge \sup_{B\in\mathcal{B}} \widehat{\operatorname{ECE}}_{B}(f).
\end{equation}
\end{proposition}
As established in \citet{Ma2021a}, a calibrated model can be constructed using a binary classifier that separates correct from incorrect predictions. We reinterpret this through the lens of selective classification, yielding the following constructive result.
\begin{proposition}[Optimal Calibration via Selective Fallback {\citep[Prop.~3]{Ma2021a}}]\label{prop:constructece}
Let \( g \) be a CSF that perfectly distinguishes correct from incorrect predictions via threshold \( \tau \). For a model \( f: \mathcal{X} \rightarrow \Delta^{k-1} \), we can construct a modified model \( f' \) such that 
\begin{align}
\sup_{B \in \mathcal{B}} \widehat{\operatorname{ECE}}_B(f') \le \sup_{B \in \mathcal{B}} \widehat{\operatorname{ECE}}_B(f), \nonumber
\end{align} where the rectified prediction \( f'_{c}(\boldsymbol{x}) \) is:
\begin{equation}
f'_{c}(\boldsymbol{x}) = 
\begin{cases}
\mathbb{I}(c=y') & \text{if } g(\boldsymbol{x}) \ge \tau, \\
1/k & \text{if } g(\boldsymbol{x}) < \tau,
\end{cases}
\end{equation}
and \( y' \in \{1, \dots, k\} \) denotes the ground-truth label.
\end{proposition}
This proposition suggests that ECE is minimized by polarizing confidence: maximizing it for correct samples while collapsing it to a uniform distribution for incorrect ones. Notably, the fallback to a uniform prediction when \( g(\boldsymbol{x}) < \tau \) mirrors the abstention mechanism in selective classification. This logic can be extended naturally to cwECE, which is formally defined as the arithmetic mean of marginal calibration errors. By reinterpreting the class-wise problem as a set of $k$ independent binary calibration tasks, we can apply the optimal rectification rule to each marginal component. Consequently, utilizing CSFs to refine the binary calibration of each class-specific margin minimizes the aggregate class-wise error as a principled consequence of optimizing its constituent parts.

 \subsection{Selective Risk Loss} 
Following Prop.~\ref{prop:ece-lb}, minimizing calibration error necessitates suppressing confidence for incorrect predictions while enhancing it for correct ones. While the latter is naturally addressed by the loss function i.e. cross-entropy during training, the former remains under-explored in trainable calibration, and this is the issue we aim to address in this paper from the perspective of selective classification. \citep{zhou2024novelcharacterizationpopulationarea} shows that AURC is a confidence-ranking–based metric, derived by integrating the selective risk over all confidence thresholds. It assigns extremely large weights to high-confidence examples, even though in most cases the relative ranking among such high-confidence examples is not very meaningful. To address this, we propose an AURC-like loss that restricts the focus to the uncertain region deliberately ignoring ranking noise among high-confidence examples to improve both calibration robustness and computational efficiency.

% As a reweighted risk function based on confidence ranking~\citep{zhou2024novelcharacterizationpopulationarea}, AURC tends to over-prioritize high-confidence samples; however, the relative ordering in these certain regions is typically uninformative for calibration. To address this, we propose an AURC-like loss that restricts the focus to the uncertain region deliberately ignoring ranking noise among high-confidence examples to improve both calibration robustness and computational efficiency.

Let us denote the low confidence region \(\mathcal{X}_\tau\) as $\mathcal{X}_\tau = \{ \boldsymbol{x} \mid g(\boldsymbol{x}) < \tau \}.$ Then, we focus on the selective risk corresponds to the samples \(\tilde{x}\) drawn from the distribution
\begin{equation}
\tilde{\boldsymbol{x}} \sim \mathbb{P}(\boldsymbol{x} \mid g(\boldsymbol{x}) < \tau) := \mathbb{P}_{\mathcal{X}_\tau},
\end{equation}
which leads to the selective area under risk coverage curve:
\begin{equation} \label{def:AUdef} 
\text{AU}_{\mathbb{P}_{\mathcal{X}_\tau}}(f) = \mathbb{E}_{\tilde{\boldsymbol{x}} \sim \mathbb{P}_{\mathcal{X}_\tau}} \left[ \frac{ \mathbb{E}_{\mathbb{P}(x, y)} \left[ \ell(f(\boldsymbol{x}), \boldsymbol{y}) \mathbb{I}[g(\boldsymbol{x}) \ge g(\tilde{\boldsymbol{x}})] \right] }{ \mathbb{E}_{\boldsymbol{x}' \sim \mathbb{P}_{x}} \mathbb{I}[g(\boldsymbol{x}') \ge g(\tilde{\boldsymbol{x}})] } \right].
\end{equation}Noticing that the expectation in the numerator can be swapped with the expectation outside, the equation above can then be written as:
\begin{align} \label{eq:pupulationaurc1}
    \text{AU}_{\mathbb{P}_{\mathcal{X}_\tau}}(f) =  \mathbb{E}_{\mathbb{P}(\boldsymbol{x}, \boldsymbol{y})} \left[ \beta(\boldsymbol{x}) \ell\left(f \left(\boldsymbol{x}\right),\boldsymbol{y}\right)         \right]
\end{align}
where 
\begin{align} \label{eq:defalpha}
\beta(\boldsymbol{x}) =  \mathbb{E}_{\tilde{\boldsymbol{x}}\sim \mathbb{P}_{\mathcal{X}_\tau}} \left(\frac{\mathbb{I}\left[g(\boldsymbol{x}) \ge g(\tilde{\boldsymbol{x}})\right] }{ \mathbb{E}_{\boldsymbol{x}^\prime \sim \mathbb{P}(\boldsymbol{x} )} \mathbb{I}\left[g(\boldsymbol{x}^\prime) \ge g(\tilde{\boldsymbol{x}})\right]} \right) 
\end{align}

\begin{proposition}[Equivalent Expression for $\beta(\boldsymbol{x})$] ~\label{prop:beta}
Let $G(s) = \Pr(g(\boldsymbol{x}) \le s)$ be the CDF of the confidence scores, then the weighting function $\beta(\boldsymbol{x})$ defined in Eq.~\eqref{eq:defalpha} admits the following closed-form expression:
\begin{equation} \label{eq:equivalentbeta}
\beta(\boldsymbol{x}) =
\begin{cases}
-C \ln(1 - G(\tau)) & \text{if } s \ge \tau, \\
-C \ln(1 - G(s)) & \text{if } s< \tau.
\end{cases}
\end{equation}
where $C= \frac{1}{G(\tau)}$ and  $ s=g(\boldsymbol{x})$ is the confidence score.
\end{proposition} 
\begin{proof}
See Appendix~\ref{proof:beta}.
\end{proof}
This formulation implies that samples within the high-confidence region ($s \ge \tau$) are weighted equally, while samples in the uncertain region are prioritized according to their relative rank in the confidence distribution. 

\noindent\textbf{[Bin-based CDF Approximation.]} Let $s_i = g(\boldsymbol{x}_i)$ denote the confidence scores. We approximate the CDF $G(s)$ in~\eqref{eq:cumulativeCSF} using a soft binning approach. Given bin edges \(\mathcal{E} = \{e_0, e_1, \ldots, e_M\}\) with \(e_0 < \cdots < e_M\), the corresponding bin centers are defined as:
\begin{equation}
c_m = \frac{e_m + e_{m+1}}{2}, \quad m = 0, \dots, M-1.
\end{equation}
For a set of confidence scores \(\{s_i\}_{i=1}^n\), let \(h_m\) denote the sample count in bin $I_m = [e_m,e_{m+1})$: $h_m = \sum_{i=1}^n \mathbb{I}[e_m \le g(\boldsymbol{x}_i) < e_{m+1}]$. To ensure differentiability for gradient-based optimization, we approximate the indicator function using a sigmoid-based smoothed step function. The estimated CDF at any score $s$ is:
\vspace{-1mm}
\begin{equation}
\widehat{G}_\nu(s) = \frac{1}{n} \sum_{m=0}^{M-1} h_m \, \sigma\!\left(\frac{s - c_m}{\nu}\right), 
\quad \sigma(t) = \frac{1}{1 + e^{-t}},
\label{eq:softcdf}
\end{equation}where \(\nu > 0\) is the smoothing parameter. A detailed discussion of this parameter is provided in Appendix~\ref{appsec:auloss}. \\

\noindent\textbf{[Finite-Sample Estimator.]} Given a threshold \(\tau\) defining the low-confidence region \(\mathcal{X}_\tau\), we  $\widehat{\beta}_i$ by substituting the smoothed CDF $\widehat{G}_\nu(s)$ into Eq.~\eqref{eq:equivalentbeta} to obtain $\widehat{\beta}_i$:
\begin{equation} \label{eq:hatbeta}
\hat{\beta}(\boldsymbol{x}) =
\begin{cases}
-C \ln(1 - \widehat{G}_\nu(\tau)) & \text{if } s_i \ge \tau, \\
-C \ln(1 - \widehat{G}_\nu(s_i)) & \text{if } s_i< \tau.
\end{cases}
\end{equation}
where $C= \frac{1}{\widehat{G}_\nu(\tau)}$ and  $ s_i=g(\boldsymbol{x}_i)$ is the confidence score. Then substituting $\widehat{\beta}_i$ into Eq.~\eqref{eq:pupulationaurc1} yields the differentiable finite-sample estimator:
\begin{equation} 
\widehat{\mathrm{AU}}_{\mathbb{P}_{\mathcal{X}_\tau}}(f) = \frac{1}{n} \sum_{i=1}^n \widehat{\beta}_i \cdot \ell \big(f(\boldsymbol{x}_i), \boldsymbol{y}_i\big).
\label{eq:auloss}
\end{equation}
\vspace{-5mm}
\begin{algorithm}[!htbp]
\caption{Optimization with $\widehat{\mathrm{AU}}_{\mathbb{P}_{\mathcal{X}_\tau}}$ loss}
\label{alg:au_update}
\begin{algorithmic}[1]
\STATE \textbf{Input:} Training set \( \mathcal{D} = \{(\boldsymbol{x}_i,\boldsymbol{y}_i)\}_{i=1}^n \), model \( f_\theta \), CSF \( g(\cdot) \), bin edges \(\mathcal{E}\), smoothing parameter \(\nu\), low-confidence threshold \(\tau\), learning rate \(\eta\)
\STATE \textbf{Forward Pass:}
\STATE \quad Compute \( \boldsymbol{p}_i = f_{\theta}(\boldsymbol{x}_i) \) and $ s_i = g(\boldsymbol{p}_i)$ 
\STATE \quad Estimate $\widehat{G}_\nu(\cdot)$ using current batch via Eq.~\eqref{eq:softcdf}
\STATE \quad Compute \( \widehat{\mathrm{AU}}_{\mathbb{P}_{\mathcal{X}_\tau}}\) via Eq.~\eqref{eq:auloss}
\STATE \textbf{Backward Pass:}
\FOR{each \( (\boldsymbol{x}_i, \boldsymbol{y}_i) \) in batch}
\STATE Compute \( \nabla_{\boldsymbol{p}_i} \widehat{\mathrm{AU}}_{\mathbb{P}_{\mathcal{X}_\tau}} \) via Eq.~\eqref{eq:gradientAU}
    \STATE \( \nabla_\theta \widehat{\mathrm{AU}}_{\mathbb{P}_{\mathcal{X}_\tau}}\mathrel{+}= (\nabla_{\boldsymbol{p}_i} \widehat{\mathrm{AU}}_{\mathbb{P}_{\mathcal{X}_\tau}})^\top \nabla_\theta \boldsymbol{p}_i \)
\ENDFOR
\STATE \textbf{Parameter Update:} \( \theta \gets \theta - \eta \nabla_\theta \widehat{\mathrm{AU}}_{\mathbb{P}_{\mathcal{X}_\tau}} \)
\end{algorithmic}
\end{algorithm}
% \vspace{-1mm}

Unlike the original AURC computation, which requires sorting all $n$ samples based on the CSF and thus incurs a complexity of $\mathcal{O}(n \log n)$, our proposed bin-based loss bypasses the sorting bottleneck. Instead, it approximates the CDF of confidence scores using $M$ bins, resulting in a computational cost of $\mathcal{O}(nM)$. Furthermore, our method remains more scalable than kernel-based calibration objectives, which typically suffer from quadratic $\mathcal{O}(n^2)$ complexity. As with AURC, the efficacy of this loss is coupled with the choice of CSF $g$, as we show in later sections, different CSFs may lead to different calibration performance. Crucially, distinct from the point-wise weighting in Focal Loss, our proposed loss determines the importance $\beta(\boldsymbol{x})$ via the relative rank of confidence scores within the distribution $G(s)$. This rank-based prioritization ensures that the loss remains robust to variations in score scales and focuses the optimization on the most critical uncertain regions.

\section{Experiments}

\begin{table*}[!htbp]
\caption{Mean $\widehat{\mathrm{cwECE}}$ $\times 10^{2}$ ($\downarrow$) for original models and after applying post-hoc calibration methods including TS. Temperature values are included in parentheses, and the temperature $T$ is selected by minimizing ECE.}
\centering
\renewcommand{\arraystretch}{1.10}
\setlength{\tabcolsep}{3pt}
\resizebox{1.0\textwidth}{!}{%
\begin{tabular}{@{}llcccccccc@{}}
\toprule
\textbf{Dataset} & \textbf{Model}       & \textbf{XE} & \textbf{FL-53} & \textbf{Inv. FL} & \textbf{Dual} & \textbf{KDE-XE} & \textbf{AURC} & \textbf{$\mathbf{\text{AU}_{\mathbb{P}_{\mathcal{X}_\tau}}}$} \\
\midrule
\rowcolor{group1} \multicolumn{1}{l}{\textbf{CIFAR10}}   & ResNet50 & 0.665  & 0.881  & 0.849  & \textbf{0.478}  & 0.677  & 0.575  & 0.581 \\
\rowcolor{group1} \multicolumn{1}{l}{{}} & \hspace{1em}+TS    & 0.597  (1.1) & 0.497  (0.9) & 0.778  (1.2) & \textbf{0.466}  (1.0) & 0.605  (1.1) & 0.577  (1.0) & 0.495  (1.2) \\
\rowcolor{group2} \multicolumn{1}{l}{{}} & ResNet110         & 0.668  & 0.933  & 1.009  & 0.620  & 0.618  & 0.558  & \textbf{0.518} \\
\rowcolor{group2} \multicolumn{1}{l}{{}} & \hspace{1em}+TS    & 0.604  (1.1) & 0.514  (0.8) & 0.914  (1.3) & 0.473  (0.9) & 0.574  (1.1) & 0.587  (1.0) & \textbf{0.439}  (1.1) \\
\rowcolor{group3} \multicolumn{1}{l}{{}} & WRN               & 0.472  & 1.231  & 0.534  & 0.899  & 0.467  & 0.489  & \textbf{0.432}  \\
\rowcolor{group3} \multicolumn{1}{l}{{}} & \hspace{1em}+TS    & 0.368  (1.2) & 0.734  (0.9) & 0.446  (1.2) & 0.758  (1.0) & 0.368  (1.2) & 0.438  (1.1) & \textbf{0.340}  (1.3) \\
\rowcolor{group4} \multicolumn{1}{l}{{}} & PreResNet56       & 0.611  & 1.176  & 0.824  & 0.577  & 0.611  & 0.611  & \textbf{0.363}\\
\rowcolor{group4} \multicolumn{1}{l}{{}} & \hspace{1em}+TS    & 0.592  (1.0) & 0.407  (0.7) & 0.805  (1.1) & 0.448  (0.9) & 0.592  (1.0) & 0.643  (1.0) & \textbf{0.335}  (1.1)\\
\midrule
\rowcolor{group1} \textbf{CIFAR100} & ResNet50           & 0.235  & 0.215  & 0.272  & 0.213  & 0.252  & 0.237  & \textbf{0.180} \\
\rowcolor{group1}                   & \hspace{1em}+TS    & 0.196  (1.3) & 0.182  (0.9) & 0.204  (1.4) & 0.189  (0.9) & 0.205  (1.4) & 0.193  (1.4) & \textbf{0.175}  (1.1) \\
\rowcolor{group2}                  & ResNet110         & 0.233  & 0.207  & 0.270  & 0.212  & 0.243  & 0.228  & \textbf{0.175}  \\
\rowcolor{group2}                  & \hspace{1em}+TS    &0.198  (1.3) & 0.181  (0.9) & 0.199  (1.4) & 0.181  (0.9) & 0.196  (1.4) & 0.183  (1.3) & \textbf{0.177}  (1.1)  \\
\rowcolor{group3}                  & WRN               &0.202  & 0.236  & 0.209  & 0.211  & 0.206  & 0.198  & \textbf{0.192}  \\
\rowcolor{group3}                  & \hspace{1em}+TS    & 0.194  (1.2) & \textbf{0.185}  (0.8) & 0.205  (1.3) & 0.189  (0.9) & 0.197  (1.2) & 0.196  (1.2) & 0.196  (1.1)  \\
\rowcolor{group4}                  & PreResNet56       & 0.233  & 0.208  & 0.267  & 0.190  & 0.232  & 0.198  & \textbf{0.177}  \\
\rowcolor{group4}                  & \hspace{1em}+TS    &0.179  (1.4) & 0.206  (0.9) & 0.182  (1.4) & 0.190  (1.0) & 0.178  (1.4) & 0.182  (1.2) & \textbf{0.175}  (1.1)  \\
\midrule
\rowcolor{group1} \multicolumn{1}{l}{{\textbf{Tiny-ImageNet}}} & ViT-Small      & 0.126  & 0.143  & 0.129  & 0.128  & 0.123  & \textbf{0.120}  & 0.121  \\
\rowcolor{group1} \multicolumn{1}{l}{{}} & \hspace{1em}+TS  &  0.121  (1.2) & 0.135  (0.9) & 0.121  (1.3) & 0.129  (1.1) & 0.121  (1.1) & 0.121  (1.0) & \textbf{0.120}  (1.0)   \\
\rowcolor{group2} \multicolumn{1}{l}{{}} & Swin-Tiny       & 0.109  & 0.155  & \textbf{0.108}  & 0.113  & 0.109  & 0.125  & 0.116  \\
\rowcolor{group2} \multicolumn{1}{l}{{}} & \hspace{1em}+TS  & 0.110  (1.0) & 0.119  (0.8) & \textbf{0.109}  (1.1) & 0.111  (1.0) & 0.110  (1.1) & 0.123  (0.9) & 0.114  (1.0) \\
\bottomrule
\end{tabular}}
% \vspace{-2mm}
\label{tab:cwece_ts_split}
\end{table*}

\begin{figure*}[!htbp]
\centering
\begin{minipage}[b]{0.16\linewidth}
  \centering
  \resizebox{\linewidth}{!}{%
    % This file was created by tikzplotlib v0.9.8.
\begin{tikzpicture}

\definecolor{color0}{rgb}{0.12156862745098,0.466666666666667,0.705882352941177}

\begin{groupplot}[group style={group size=1 by 2, vertical sep=2cm}]
\nextgroupplot[
axis background/.style={fill=white!89.8039215686275!black},
axis line style={white},
legend cell align={left},
legend style={
  fill opacity=0.8,
  draw opacity=1,
  text opacity=1,
  at={(0.03,0.97)},
  anchor=north west,
  draw=white!80!black,
  fill=white!89.8039215686275!black
},
tick align=outside,
tick pos=left,
x grid style={white},
xlabel={Confidence},
xmajorgrids,
xmin=0, xmax=1,
xtick style={color=white!33.3333333333333!black},
xtick={0,0.2,0.4,0.6,0.8,1},
xticklabels={0.0,0.2,0.4,0.6,0.8,1.0},
y grid style={white},
ylabel={Accuracy},
ymajorgrids,
ymin=0, ymax=1,
ytick style={color=white!33.3333333333333!black},
ytick={0,0.2,0.4,0.6,0.8,1},
yticklabels={0.0,0.2,0.4,0.6,0.8,1.0}
]
\draw[draw=black,fill=color0,very thin] (axis cs:-6.93889390390723e-18,0) rectangle (axis cs:0.1,0.05);
\draw[draw=black,fill=color0,very thin] (axis cs:0.1,0) rectangle (axis cs:0.2,0.15);
\draw[draw=black,fill=color0,very thin] (axis cs:0.2,0) rectangle (axis cs:0.3,0.263157894736842);
\draw[draw=black,fill=color0,very thin] (axis cs:0.3,0) rectangle (axis cs:0.4,0.423076923076923);
\draw[draw=black,fill=color0,very thin] (axis cs:0.4,0) rectangle (axis cs:0.5,0.560714285714286);
\draw[draw=black,fill=color0,very thin] (axis cs:0.5,0) rectangle (axis cs:0.6,0.634191176470588);
\draw[draw=black,fill=color0,very thin] (axis cs:0.6,0) rectangle (axis cs:0.7,0.74468085106383);
\draw[draw=black,fill=color0,very thin] (axis cs:0.7,0) rectangle (axis cs:0.8,0.881578947368421);
\draw[draw=black,fill=color0,very thin] (axis cs:0.8,0) rectangle (axis cs:0.9,0.967201674808095);
\draw[draw=black,fill=color0,very thin] (axis cs:0.9,0) rectangle (axis cs:1,0.996346886912325);
\draw[draw=black,fill=red,opacity=0.6,very thin] (axis cs:-6.93889390390723e-18,0.05) rectangle (axis cs:0.1,0.05);
\draw[draw=black,fill=red,opacity=0.6,very thin] (axis cs:0.1,0.15) rectangle (axis cs:0.2,0.15);
\draw[draw=black,fill=red,opacity=0.6,very thin] (axis cs:0.2,0.263157894736842) rectangle (axis cs:0.3,0.275920282853277);
\draw[draw=black,fill=red,opacity=0.6,very thin] (axis cs:0.3,0.423076923076923) rectangle (axis cs:0.4,0.361916334583209);
\draw[draw=black,fill=red,opacity=0.6,very thin] (axis cs:0.4,0.560714285714286) rectangle (axis cs:0.5,0.456773817219905);
\draw[draw=black,fill=red,opacity=0.6,very thin] (axis cs:0.5,0.634191176470588) rectangle (axis cs:0.6,0.551230576327618);
\draw[draw=black,fill=red,opacity=0.6,very thin] (axis cs:0.6,0.74468085106383) rectangle (axis cs:0.7,0.651467973458851);
\draw[draw=black,fill=red,opacity=0.6,very thin] (axis cs:0.7,0.881578947368421) rectangle (axis cs:0.8,0.752674729416245);
\draw[draw=black,fill=red,opacity=0.6,very thin] (axis cs:0.8,0.967201674808095) rectangle (axis cs:0.9,0.858105590604688);
\draw[draw=black,fill=red,opacity=0.6,very thin] (axis cs:0.9,0.996346886912325) rectangle (axis cs:1,0.96540744067791);
\addplot [semithick, red, dash pattern=on 5.55pt off 2.4pt]
table {%
0 0
1 1
};
\addlegendentry{Perfect Calibration}
\addlegendimage{area legend,draw=black, fill=color0, mark size=3}
\addlegendentry{Output}
\addlegendimage{area legend, draw=black, fill=red, opacity=0.6}
\addlegendentry{Gap}
\end{groupplot}

\draw ({$(current bounding box.south west)!0.95!(current bounding box.south east)$}|-{$(current bounding box.south west)!0.18!(current bounding box.north west)$}) node[
  scale=1.2,
  fill=white,
  draw=none,
  line width=0.4pt,
  inner sep=2.9pt,
  fill opacity=0.9,
  anchor=south east,
  text=black,
  rotate=0.0
]{ECE=5.83};
\end{tikzpicture}
  }
  \subcaption{FL-53}
\end{minipage}
\begin{minipage}[b]{0.16\linewidth}
  \centering
  \resizebox{\linewidth}{!}{%
    % This file was created by tikzplotlib v0.9.8.
\begin{tikzpicture}

\definecolor{color0}{rgb}{0.12156862745098,0.466666666666667,0.705882352941177}

\begin{groupplot}[group style={group size=1 by 2, vertical sep=2cm}]
\nextgroupplot[
axis background/.style={fill=white!89.8039215686275!black},
axis line style={white},
legend cell align={left},
legend style={
  fill opacity=0.8,
  draw opacity=1,
  text opacity=1,
  at={(0.03,0.97)},
  anchor=north west,
  draw=white!80!black,
  fill=white!89.8039215686275!black
},
tick align=outside,
tick pos=left,
x grid style={white},
xlabel={Confidence},
xmajorgrids,
xmin=0, xmax=1,
xtick style={color=white!33.3333333333333!black},
xtick={0,0.2,0.4,0.6,0.8,1},
xticklabels={0.0,0.2,0.4,0.6,0.8,1.0},
y grid style={white},
ylabel={Accuracy},
ymajorgrids,
ymin=0, ymax=1,
ytick style={color=white!33.3333333333333!black},
ytick={0,0.2,0.4,0.6,0.8,1},
yticklabels={0.0,0.2,0.4,0.6,0.8,1.0}
]
\draw[draw=black,fill=color0,very thin] (axis cs:-6.93889390390723e-18,0) rectangle (axis cs:0.1,0.05);
\draw[draw=black,fill=color0,very thin] (axis cs:0.1,0) rectangle (axis cs:0.2,0.15);
\draw[draw=black,fill=color0,very thin] (axis cs:0.2,0) rectangle (axis cs:0.3,0);
\draw[draw=black,fill=color0,very thin] (axis cs:0.3,0) rectangle (axis cs:0.4,0.3);
\draw[draw=black,fill=color0,very thin] (axis cs:0.4,0) rectangle (axis cs:0.5,0.384615384615385);
\draw[draw=black,fill=color0,very thin] (axis cs:0.5,0) rectangle (axis cs:0.6,0.378048780487805);
\draw[draw=black,fill=color0,very thin] (axis cs:0.6,0) rectangle (axis cs:0.7,0.488636363636364);
\draw[draw=black,fill=color0,very thin] (axis cs:0.7,0) rectangle (axis cs:0.8,0.516393442622951);
\draw[draw=black,fill=color0,very thin] (axis cs:0.8,0) rectangle (axis cs:0.9,0.565217391304348);
\draw[draw=black,fill=color0,very thin] (axis cs:0.9,0) rectangle (axis cs:1,0.968163609529833);
\draw[draw=black,fill=red,opacity=0.6,very thin] (axis cs:-6.93889390390723e-18,0.05) rectangle (axis cs:0.1,0.05);
\draw[draw=black,fill=red,opacity=0.6,very thin] (axis cs:0.1,0.15) rectangle (axis cs:0.2,0.15);
\draw[draw=black,fill=red,opacity=0.6,very thin] (axis cs:0.2,0) rectangle (axis cs:0.3,0.290134996175766);
\draw[draw=black,fill=red,opacity=0.6,very thin] (axis cs:0.3,0.3) rectangle (axis cs:0.4,0.368676760792732);
\draw[draw=black,fill=red,opacity=0.6,very thin] (axis cs:0.4,0.384615384615385) rectangle (axis cs:0.5,0.460584498368777);
\draw[draw=black,fill=red,opacity=0.6,very thin] (axis cs:0.5,0.378048780487805) rectangle (axis cs:0.6,0.554585520087219);
\draw[draw=black,fill=red,opacity=0.6,very thin] (axis cs:0.6,0.488636363636364) rectangle (axis cs:0.7,0.649214595556259);
\draw[draw=black,fill=red,opacity=0.6,very thin] (axis cs:0.7,0.516393442622951) rectangle (axis cs:0.8,0.752639370863555);
\draw[draw=black,fill=red,opacity=0.6,very thin] (axis cs:0.8,0.565217391304348) rectangle (axis cs:0.9,0.855443768527197);
\draw[draw=black,fill=red,opacity=0.6,very thin] (axis cs:0.9,0.968163609529833) rectangle (axis cs:1,0.997481431453107);
\addplot [semithick, red, dash pattern=on 5.55pt off 2.4pt]
table {%
0 0
1 1
};
\addlegendentry{Perfect Calibration}
\addlegendimage{area legend,draw=black, fill=color0, mark size=3}
\addlegendentry{Output}
\addlegendimage{area legend, draw=black, fill=red, opacity=0.6}
\addlegendentry{Gap}
\end{groupplot}

\draw ({$(current bounding box.south west)!0.95!(current bounding box.south east)$}|-{$(current bounding box.south west)!0.18!(current bounding box.north west)$}) node[
  scale=1.2,
  fill=white,
  draw=none,
  line width=0.4pt,
  inner sep=2.9pt,
  fill opacity=0.9,
  anchor=south east,
  text=black,
  rotate=0.0
]{ECE=3.92};
\end{tikzpicture}
  }
  \subcaption{Inverse Focal}
\end{minipage}
\begin{minipage}[b]{0.16\linewidth}
  \centering
  \resizebox{\linewidth}{!}{%
    % This file was created by tikzplotlib v0.9.8.
\begin{tikzpicture}

\definecolor{color0}{rgb}{0.12156862745098,0.466666666666667,0.705882352941177}

\begin{groupplot}[group style={group size=1 by 2, vertical sep=2cm}]
\nextgroupplot[
axis background/.style={fill=white!89.8039215686275!black},
axis line style={white},
legend cell align={left},
legend style={
  fill opacity=0.8,
  draw opacity=1,
  text opacity=1,
  at={(0.03,0.97)},
  anchor=north west,
  draw=white!80!black,
  fill=white!89.8039215686275!black
},
tick align=outside,
tick pos=left,
x grid style={white},
xlabel={Confidence},
xmajorgrids,
xmin=0, xmax=1,
xtick style={color=white!33.3333333333333!black},
xtick={0,0.2,0.4,0.6,0.8,1},
xticklabels={0.0,0.2,0.4,0.6,0.8,1.0},
y grid style={white},
ylabel={Accuracy},
ymajorgrids,
ymin=0, ymax=1,
ytick style={color=white!33.3333333333333!black},
ytick={0,0.2,0.4,0.6,0.8,1},
yticklabels={0.0,0.2,0.4,0.6,0.8,1.0}
]
\draw[draw=black,fill=color0,very thin] (axis cs:-6.93889390390723e-18,0) rectangle (axis cs:0.1,0.05);
\draw[draw=black,fill=color0,very thin] (axis cs:0.1,0) rectangle (axis cs:0.2,0.15);
\draw[draw=black,fill=color0,very thin] (axis cs:0.2,0) rectangle (axis cs:0.3,0.352941176470588);
\draw[draw=black,fill=color0,very thin] (axis cs:0.3,0) rectangle (axis cs:0.4,0.449438202247191);
\draw[draw=black,fill=color0,very thin] (axis cs:0.4,0) rectangle (axis cs:0.5,0.461538461538462);
\draw[draw=black,fill=color0,very thin] (axis cs:0.5,0) rectangle (axis cs:0.6,0.559375);
\draw[draw=black,fill=color0,very thin] (axis cs:0.6,0) rectangle (axis cs:0.7,0.70189701897019);
\draw[draw=black,fill=color0,very thin] (axis cs:0.7,0) rectangle (axis cs:0.8,0.783898305084746);
\draw[draw=black,fill=color0,very thin] (axis cs:0.8,0) rectangle (axis cs:0.9,0.91504854368932);
\draw[draw=black,fill=color0,very thin] (axis cs:0.9,0) rectangle (axis cs:1,0.99210560372719);
\draw[draw=black,fill=red,opacity=0.6,very thin] (axis cs:-6.93889390390723e-18,0.05) rectangle (axis cs:0.1,0.05);
\draw[draw=black,fill=red,opacity=0.6,very thin] (axis cs:0.1,0.15) rectangle (axis cs:0.2,0.15);
\draw[draw=black,fill=red,opacity=0.6,very thin] (axis cs:0.2,0.352941176470588) rectangle (axis cs:0.3,0.26837073441814);
\draw[draw=black,fill=red,opacity=0.6,very thin] (axis cs:0.3,0.449438202247191) rectangle (axis cs:0.4,0.357843224587065);
\draw[draw=black,fill=red,opacity=0.6,very thin] (axis cs:0.4,0.461538461538462) rectangle (axis cs:0.5,0.45631013417637);
\draw[draw=black,fill=red,opacity=0.6,very thin] (axis cs:0.5,0.559375) rectangle (axis cs:0.6,0.549521796032786);
\draw[draw=black,fill=red,opacity=0.6,very thin] (axis cs:0.6,0.70189701897019) rectangle (axis cs:0.7,0.65349736753195);
\draw[draw=black,fill=red,opacity=0.6,very thin] (axis cs:0.7,0.783898305084746) rectangle (axis cs:0.8,0.755163210554648);
\draw[draw=black,fill=red,opacity=0.6,very thin] (axis cs:0.8,0.91504854368932) rectangle (axis cs:0.9,0.857545783942186);
\draw[draw=black,fill=red,opacity=0.6,very thin] (axis cs:0.9,0.99210560372719) rectangle (axis cs:1,0.980021750863718);
\addplot [semithick, red, dash pattern=on 5.55pt off 2.4pt]
table {%
0 0
1 1
};
\addlegendentry{Perfect Calibration}
\addlegendimage{area legend,draw=black, fill=color0, mark size=3}
\addlegendentry{Output}
\addlegendimage{area legend, draw=black, fill=red, opacity=0.6}
\addlegendentry{Gap}
\end{groupplot}

\draw ({$(current bounding box.south west)!0.95!(current bounding box.south east)$}|-{$(current bounding box.south west)!0.18!(current bounding box.north west)$}) node[
  scale=1.2,
  fill=white,
  draw=none,
  line width=0.4pt,
  inner sep=2.9pt,
  fill opacity=0.9,
  anchor=south east,
  text=black,
  rotate=0.0
]{ECE=1.92};
\end{tikzpicture}
  }
  \subcaption{Dual Focal}
\end{minipage}
\begin{minipage}[b]{0.16\linewidth}
  \centering
  \resizebox{\linewidth}{!}{%
    % This file was created by tikzplotlib v0.9.8.
\begin{tikzpicture}

\definecolor{color0}{rgb}{0.12156862745098,0.466666666666667,0.705882352941177}

\begin{groupplot}[group style={group size=1 by 2, vertical sep=2cm}]
\nextgroupplot[
axis background/.style={fill=white!89.8039215686275!black},
axis line style={white},
legend cell align={left},
legend style={
  fill opacity=0.8,
  draw opacity=1,
  text opacity=1,
  at={(0.03,0.97)},
  anchor=north west,
  draw=white!80!black,
  fill=white!89.8039215686275!black
},
tick align=outside,
tick pos=left,
x grid style={white},
xlabel={Confidence},
xmajorgrids,
xmin=0, xmax=1,
xtick style={color=white!33.3333333333333!black},
xtick={0,0.2,0.4,0.6,0.8,1},
xticklabels={0.0,0.2,0.4,0.6,0.8,1.0},
y grid style={white},
ylabel={Accuracy},
ymajorgrids,
ymin=0, ymax=1,
ytick style={color=white!33.3333333333333!black},
ytick={0,0.2,0.4,0.6,0.8,1},
yticklabels={0.0,0.2,0.4,0.6,0.8,1.0}
]
\draw[draw=black,fill=color0,very thin] (axis cs:-6.93889390390723e-18,0) rectangle (axis cs:0.1,0.05);
\draw[draw=black,fill=color0,very thin] (axis cs:0.1,0) rectangle (axis cs:0.2,0.15);
\draw[draw=black,fill=color0,very thin] (axis cs:0.2,0) rectangle (axis cs:0.3,0.4);
\draw[draw=black,fill=color0,very thin] (axis cs:0.3,0) rectangle (axis cs:0.4,0.217391304347826);
\draw[draw=black,fill=color0,very thin] (axis cs:0.4,0) rectangle (axis cs:0.5,0.364705882352941);
\draw[draw=black,fill=color0,very thin] (axis cs:0.5,0) rectangle (axis cs:0.6,0.492146596858639);
\draw[draw=black,fill=color0,very thin] (axis cs:0.6,0) rectangle (axis cs:0.7,0.60377358490566);
\draw[draw=black,fill=color0,very thin] (axis cs:0.7,0) rectangle (axis cs:0.8,0.686695278969957);
\draw[draw=black,fill=color0,very thin] (axis cs:0.8,0) rectangle (axis cs:0.9,0.730769230769231);
\draw[draw=black,fill=color0,very thin] (axis cs:0.9,0) rectangle (axis cs:1,0.978332016702404);
\draw[draw=black,fill=red,opacity=0.6,very thin] (axis cs:-6.93889390390723e-18,0.05) rectangle (axis cs:0.1,0.05);
\draw[draw=black,fill=red,opacity=0.6,very thin] (axis cs:0.1,0.15) rectangle (axis cs:0.2,0.15);
\draw[draw=black,fill=red,opacity=0.6,very thin] (axis cs:0.2,0.4) rectangle (axis cs:0.3,0.276703956723213);
\draw[draw=black,fill=red,opacity=0.6,very thin] (axis cs:0.3,0.217391304347826) rectangle (axis cs:0.4,0.361091696697733);
\draw[draw=black,fill=red,opacity=0.6,very thin] (axis cs:0.4,0.364705882352941) rectangle (axis cs:0.5,0.455720301235423);
\draw[draw=black,fill=red,opacity=0.6,very thin] (axis cs:0.5,0.492146596858639) rectangle (axis cs:0.6,0.550612365388121);
\draw[draw=black,fill=red,opacity=0.6,very thin] (axis cs:0.6,0.60377358490566) rectangle (axis cs:0.7,0.65028412724441);
\draw[draw=black,fill=red,opacity=0.6,very thin] (axis cs:0.7,0.686695278969957) rectangle (axis cs:0.8,0.749349552418541);
\draw[draw=black,fill=red,opacity=0.6,very thin] (axis cs:0.8,0.730769230769231) rectangle (axis cs:0.9,0.856711530838257);
\draw[draw=black,fill=red,opacity=0.6,very thin] (axis cs:0.9,0.978332016702404) rectangle (axis cs:1,0.994159726973595);
\addplot [semithick, red, dash pattern=on 5.55pt off 2.4pt]
table {%
0 0
1 1
};
\addlegendentry{Perfect Calibration}
\addlegendimage{area legend,draw=black, fill=color0, mark size=3}
\addlegendentry{Output}
\addlegendimage{area legend, draw=black, fill=red, opacity=0.6}
\addlegendentry{Gap}
\end{groupplot}

\draw ({$(current bounding box.south west)!0.95!(current bounding box.south east)$}|-{$(current bounding box.south west)!0.18!(current bounding box.north west)$}) node[
  scale=1.2,
  fill=white,
  draw=none,
  line width=0.4pt,
  inner sep=2.9pt,
  fill opacity=0.9,
  anchor=south east,
  text=black,
  rotate=0.0
]{ECE=2.38};
\end{tikzpicture}
  }
  \subcaption{KDE-XE}
\end{minipage}
\begin{minipage}[b]{0.16\linewidth}
  \centering
  \resizebox{\linewidth}{!}{%
    % This file was created by tikzplotlib v0.9.8.
\begin{tikzpicture}

\definecolor{color0}{rgb}{0.12156862745098,0.466666666666667,0.705882352941177}

\begin{groupplot}[group style={group size=1 by 2, vertical sep=2cm}]
\nextgroupplot[
axis background/.style={fill=white!89.8039215686275!black},
axis line style={white},
legend cell align={left},
legend style={
  fill opacity=0.8,
  draw opacity=1,
  text opacity=1,
  at={(0.03,0.97)},
  anchor=north west,
  draw=white!80!black,
  fill=white!89.8039215686275!black
},
tick align=outside,
tick pos=left,
x grid style={white},
xlabel={Confidence},
xmajorgrids,
xmin=0, xmax=1,
xtick style={color=white!33.3333333333333!black},
xtick={0,0.2,0.4,0.6,0.8,1},
xticklabels={0.0,0.2,0.4,0.6,0.8,1.0},
y grid style={white},
ylabel={Accuracy},
ymajorgrids,
ymin=0, ymax=1,
ytick style={color=white!33.3333333333333!black},
ytick={0,0.2,0.4,0.6,0.8,1},
yticklabels={0.0,0.2,0.4,0.6,0.8,1.0}
]
\draw[draw=black,fill=color0,very thin] (axis cs:-6.93889390390723e-18,0) rectangle (axis cs:0.1,0.05);
\draw[draw=black,fill=color0,very thin] (axis cs:0.1,0) rectangle (axis cs:0.2,0);
\draw[draw=black,fill=color0,very thin] (axis cs:0.2,0) rectangle (axis cs:0.3,0.125);
\draw[draw=black,fill=color0,very thin] (axis cs:0.3,0) rectangle (axis cs:0.4,0.282608695652174);
\draw[draw=black,fill=color0,very thin] (axis cs:0.4,0) rectangle (axis cs:0.5,0.435114503816794);
\draw[draw=black,fill=color0,very thin] (axis cs:0.5,0) rectangle (axis cs:0.6,0.55440414507772);
\draw[draw=black,fill=color0,very thin] (axis cs:0.6,0) rectangle (axis cs:0.7,0.483091787439614);
\draw[draw=black,fill=color0,very thin] (axis cs:0.7,0) rectangle (axis cs:0.8,0.592885375494071);
\draw[draw=black,fill=color0,very thin] (axis cs:0.8,0) rectangle (axis cs:0.9,0.7);
\draw[draw=black,fill=color0,very thin] (axis cs:0.9,0) rectangle (axis cs:1,0.975570957845699);
\draw[draw=black,fill=red,opacity=0.6,very thin] (axis cs:-6.93889390390723e-18,0.05) rectangle (axis cs:0.1,0.05);
\draw[draw=black,fill=red,opacity=0.6,very thin] (axis cs:0.1,0) rectangle (axis cs:0.2,0.194341078400612);
\draw[draw=black,fill=red,opacity=0.6,very thin] (axis cs:0.2,0.125) rectangle (axis cs:0.3,0.281646598130465);
\draw[draw=black,fill=red,opacity=0.6,very thin] (axis cs:0.3,0.282608695652174) rectangle (axis cs:0.4,0.355309063973634);
\draw[draw=black,fill=red,opacity=0.6,very thin] (axis cs:0.4,0.435114503816794) rectangle (axis cs:0.5,0.456321502459868);
\draw[draw=black,fill=red,opacity=0.6,very thin] (axis cs:0.5,0.55440414507772) rectangle (axis cs:0.6,0.547795350996324);
\draw[draw=black,fill=red,opacity=0.6,very thin] (axis cs:0.6,0.483091787439614) rectangle (axis cs:0.7,0.651141056979912);
\draw[draw=black,fill=red,opacity=0.6,very thin] (axis cs:0.7,0.592885375494071) rectangle (axis cs:0.8,0.751641168660326);
\draw[draw=black,fill=red,opacity=0.6,very thin] (axis cs:0.8,0.7) rectangle (axis cs:0.9,0.85556810349226);
\draw[draw=black,fill=red,opacity=0.6,very thin] (axis cs:0.9,0.975570957845699) rectangle (axis cs:1,0.993177865502672);
\addplot [semithick, red, dash pattern=on 5.55pt off 2.4pt]
table {%
0 0
1 1
};
\addlegendentry{Perfect Calibration}
\addlegendimage{area legend,draw=black, fill=color0, mark size=3}
\addlegendentry{Output}
\addlegendimage{area legend, draw=black, fill=red, opacity=0.6}
\addlegendentry{Gap}
\end{groupplot}

\draw ({$(current bounding box.south west)!0.95!(current bounding box.south east)$}|-{$(current bounding box.south west)!0.18!(current bounding box.north west)$}) node[
  scale=1.2,
  fill=white,
  draw=none,
  line width=0.4pt,
  inner sep=2.9pt,
  fill opacity=0.9,
  anchor=south east,
  text=black,
  rotate=0.0
]{ECE=2.98};
\end{tikzpicture}
  }
  \subcaption{AURC}
\end{minipage}
\begin{minipage}[b]{0.16\linewidth}
  \centering
  \resizebox{\linewidth}{!}{%
    % This file was created by tikzplotlib v0.9.8.
\begin{tikzpicture}

\definecolor{color0}{rgb}{0.12156862745098,0.466666666666667,0.705882352941177}

\begin{groupplot}[group style={group size=1 by 2, vertical sep=2cm}]
\nextgroupplot[
axis background/.style={fill=white!89.8039215686275!black},
axis line style={white},
legend cell align={left},
legend style={
  fill opacity=0.8,
  draw opacity=1,
  text opacity=1,
  at={(0.03,0.97)},
  anchor=north west,
  draw=white!80!black,
  fill=white!89.8039215686275!black
},
tick align=outside,
tick pos=left,
x grid style={white},
xlabel={Confidence},
xmajorgrids,
xmin=0, xmax=1,
xtick style={color=white!33.3333333333333!black},
xtick={0,0.2,0.4,0.6,0.8,1},
xticklabels={0.0,0.2,0.4,0.6,0.8,1.0},
y grid style={white},
ylabel={Accuracy},
ymajorgrids,
ymin=0, ymax=1,
ytick style={color=white!33.3333333333333!black},
ytick={0,0.2,0.4,0.6,0.8,1},
yticklabels={0.0,0.2,0.4,0.6,0.8,1.0}
]
\draw[draw=black,fill=color0,very thin] (axis cs:-6.93889390390723e-18,0) rectangle (axis cs:0.1,0.05);
\draw[draw=black,fill=color0,very thin] (axis cs:0.1,0) rectangle (axis cs:0.2,0.15);
\draw[draw=black,fill=color0,very thin] (axis cs:0.2,0) rectangle (axis cs:0.3,0.375);
\draw[draw=black,fill=color0,very thin] (axis cs:0.3,0) rectangle (axis cs:0.4,0.371428571428571);
\draw[draw=black,fill=color0,very thin] (axis cs:0.4,0) rectangle (axis cs:0.5,0.436046511627907);
\draw[draw=black,fill=color0,very thin] (axis cs:0.5,0) rectangle (axis cs:0.6,0.564189189189189);
\draw[draw=black,fill=color0,very thin] (axis cs:0.6,0) rectangle (axis cs:0.7,0.633986928104575);
\draw[draw=black,fill=color0,very thin] (axis cs:0.7,0) rectangle (axis cs:0.8,0.725490196078431);
\draw[draw=black,fill=color0,very thin] (axis cs:0.8,0) rectangle (axis cs:0.9,0.809433962264151);
\draw[draw=black,fill=color0,very thin] (axis cs:0.9,0) rectangle (axis cs:1,0.986065673088574);
\draw[draw=black,fill=red,opacity=0.6,very thin] (axis cs:-6.93889390390723e-18,0.05) rectangle (axis cs:0.1,0.05);
\draw[draw=black,fill=red,opacity=0.6,very thin] (axis cs:0.1,0.15) rectangle (axis cs:0.2,0.15);
\draw[draw=black,fill=red,opacity=0.6,very thin] (axis cs:0.2,0.375) rectangle (axis cs:0.3,0.269405594095588);
\draw[draw=black,fill=red,opacity=0.6,very thin] (axis cs:0.3,0.371428571428571) rectangle (axis cs:0.4,0.359135525141444);
\draw[draw=black,fill=red,opacity=0.6,very thin] (axis cs:0.4,0.436046511627907) rectangle (axis cs:0.5,0.45345939228008);
\draw[draw=black,fill=red,opacity=0.6,very thin] (axis cs:0.5,0.564189189189189) rectangle (axis cs:0.6,0.54707975242589);
\draw[draw=black,fill=red,opacity=0.6,very thin] (axis cs:0.6,0.633986928104575) rectangle (axis cs:0.7,0.649735633454292);
\draw[draw=black,fill=red,opacity=0.6,very thin] (axis cs:0.7,0.725490196078431) rectangle (axis cs:0.8,0.750686002879584);
\draw[draw=black,fill=red,opacity=0.6,very thin] (axis cs:0.8,0.809433962264151) rectangle (axis cs:0.9,0.856447905751894);
\draw[draw=black,fill=red,opacity=0.6,very thin] (axis cs:0.9,0.986065673088574) rectangle (axis cs:1,0.991281442002038);
\addplot [semithick, red, dash pattern=on 5.55pt off 2.4pt]
table {%
0 0
1 1
};
\addlegendentry{Perfect Calibration}
\addlegendimage{area legend,draw=black, fill=color0, mark size=3}
\addlegendentry{Output}
\addlegendimage{area legend, draw=black, fill=red, opacity=0.6}
\addlegendentry{Gap}
\end{groupplot}

\draw ({$(current bounding box.south west)!0.95!(current bounding box.south east)$}|-{$(current bounding box.south west)!0.18!(current bounding box.north west)$}) node[
  scale=1.2,
  fill=white,
  draw=none,
  line width=0.4pt,
  inner sep=2.9pt,
  fill opacity=0.9,
  anchor=south east,
  text=black,
  rotate=0.0
]{ECE=0.92};
\end{tikzpicture}
  }
  \subcaption{$\mathbf{\text{AU}_{P_{\mathcal{X}_\tau}}}$}
\end{minipage}
\vspace{-2mm}
\caption{Reliability diagrams of PreResNet56 on CIFAR-10. Top-label ECE (\%) appears below each panel.}
\label{fig:rc_ece}
\end{figure*}

\begin{figure*}[htbp] 
\centering 
\includegraphics[width=1.0\linewidth]{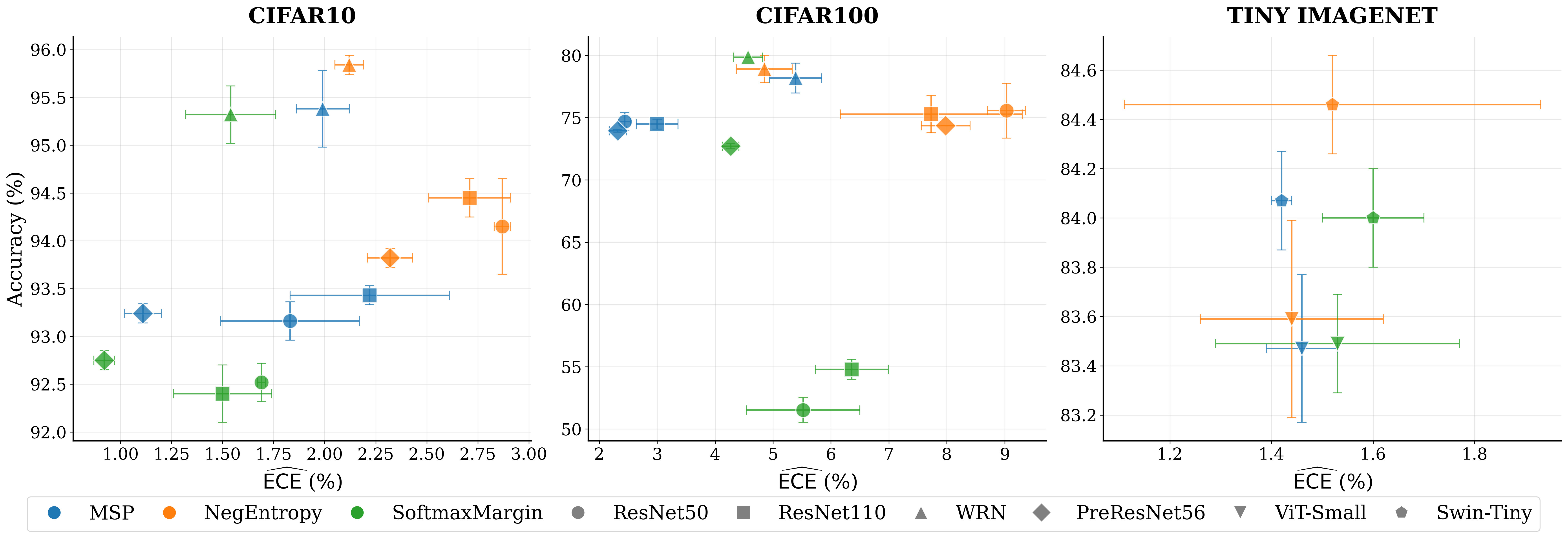}
\caption{Test $\widehat{\text{ECE}}$ of $\widehat{\mathrm{AU}}_{\mathbb{P}_{\mathcal{X}_\tau}}$ loss with different CSFs, reported as mean $\times 10^{2}$.}
 \label{fig:influenceCSF}
\end{figure*}

\begin{table*}[htbp]
\caption{Test accuracy and $\widehat{\text{ECE}}$ ($\downarrow$), reported as mean $\times 10^{2}$.}
% \vspace{-2mm}
\centering
% \scriptsize
\setlength{\tabcolsep}{1.2pt} % Increase column spacing
\renewcommand{\arraystretch}{1.15} % Increase row spacing
\resizebox{1.0\textwidth}{!}{
\begin{tabular}{llccccccc|ccccccc}
\toprule
\fontsize{9}{11}\selectfont  & \textbf{Model} & \textbf{XE} & \textbf{FL-53} & \textbf{Inv.FL} & \textbf{Dual} & \textbf{KDE-XE} & \textbf{AURC} & $\mathbf{\text{AU}_{\mathbb{P}_{\mathcal{X}_\tau}}}$ &
\textbf{XE} & \textbf{FL-53} & \textbf{Inv.FL} & \textbf{Dual} & \textbf{KDE-XE} & \textbf{AURC} & $\mathbf{\text{AU}_{\mathbb{P}_{\mathcal{X}_\tau}}}$ \\
\midrule
\multicolumn{14}{l}{\textbf{CIFAR-10} \quad \quad  \hfill  \textbf{Accuracy} ($\uparrow$) \hfill \hfill \textbf{ECE} ($\downarrow$)} \\
\rowcolor{gray!10} 
& ResNet50               & \textbf{94.51} & 93.58 & 94.36 & 94.09 & 94.46 & 93.32 & 92.52 & 3.00 & 3.98 & 4.05 & \textbf{1.27} & 3.09 & 2.53 & 1.69 \\
& ResNet110              & 94.47 & 93.80 & 93.92 & 93.72 & \textbf{94.68} & 93.05 & 92.40 & 3.04 & 4.44 & 4.80 & 2.08 & 2.80 & 2.44 & \textbf{1.50} \\
\rowcolor{gray!10} 
& WRN                    & 95.88 & 95.49 & \textbf{95.98} & 95.74 & 95.84 & 94.93 & 95.32 & 1.99 & 5.74 & 2.44 & 3.96 & 2.05 & 1.88 & \textbf{1.54} \\
& PreResNet56            & 93.97 & 92.97 & \textbf{94.14} & 93.72 & 94.09 & 92.63 & 92.75 & 2.62 & 5.50 & 3.89 & 2.07 & 2.57 & 1.75 & \textbf{0.92} \\
\midrule
\multicolumn{14}{l}{\textbf{CIFAR-100}\quad \quad \quad \hfill  \textbf{Accuracy} ($\uparrow$) \hfill\hfill \textbf{ECE} ($\downarrow$)} \\
\rowcolor{gray!10} 
& ResNet50               & 74.37 & \textbf{76.41} & 75.64 & 75.46 & 76.25 & 75.40 & 74.69 & 7.18 & 4.51 & 10.70 & 4.20 & 9.59 & 8.43 & \textbf{2.44} \\
& ResNet110              & 75.27 & 77.48 & 77.36 & 76.30 & \textbf{78.27} & 74.33 & 74.48 & 7.29 & 4.57 & 10.90 & 4.60 & 9.09 & 5.23 & \textbf{3.00} \\
\rowcolor{gray!10} 
& WRN                    & 78.68 & 79.07 & 79.46 & 78.38 & 79.63 & 76.46 & \textbf{79.85} & 5.40 & 6.22 & 6.20 & \textbf{3.54} & 5.77 & 4.75 & 4.47 \\
& PreResNet56            & 74.28 & 72.93 & 74.18 & 72.96 & \textbf{74.69} & 71.91 & 73.94 & 7.10 & 2.44 & 9.81 & \textbf{0.96} & 7.42 & 4.88 & 2.32 \\
\midrule
\multicolumn{14}{l}{\textbf{Tiny-ImageNet} \hfill \textbf{Accuracy} ($\uparrow$) \hfill \hfill \textbf{ECE} ($\downarrow$)} \\
\rowcolor{gray!10} 
& Vit-Small              & 82.94 & 81.39 & 82.87 & 82.08 & 83.26 & 83.10 & \textbf{83.59} & 4.65 & 4.14 & 6.28 & 3.19 & 3.02 &  1.44 & \textbf{1.44} \\
& Swin-Tiny              & 85.36 & 84.31 & \textbf{85.64} & 85.42 & 85.36 & 82.35 & 84.07 & 1.62 & 7.69 & 3.00 & 1.60 & 1.71 & 2.02 & \textbf{1.42} \\
\bottomrule
\end{tabular}}
\label{tab:combined_acc_ece}
\end{table*}
\textbf{Datasets and Models.} We conduct experiments on image classification datasets including CIFAR10/100~\citep{krizhevsky2009learning}, and Tiny-ImageNet~\citep{deng2009imagenet}. On CIFAR-10/100, we evaluate models including WideResNet-28x10~\citep{zagoruyko2016wide}, PreResNet-56~\citep{he2016identity}, and ResNet architectures with depths 50 and 110~\citep{he2016deep}. For Tiny-ImageNet, we use two pretrained transformer-based models from the \textit{timm} library~\citep{rw2019timm}: ViT-Small~\citep{dosovitskiy2020image} and Swin-Tiny~\citep{liu2021swin}. Each experiment is repeated with three different seeds, and we report the mean performance across these runs. 

\noindent \textbf{Experimental setup.} For experiments on CIFAR-10/100, we follow the training setup described in~\citet{xia2022moderate}. For transformer-based models on Tiny-ImageNet, we adopt the training strategy from~\citet{yoshioka2024visiontransformers}. See Appendix~\ref{appendix:trainingdetails} for details. 

\noindent \textbf{Evaluation Metrics.} We evaluate calibration performance on the test set using binned estimators of ECE and cwECE, computed with two binning strategies: equal-width binning (denoted as $\widehat{\text{ECE}}$ and $\widehat{\text{cwECE}}$) and adaptive binning~\citep{ding2020revisiting} (denoted as $\widehat{\text{ECE}}^a$ and $\widehat{\text{cwECE}}^a$). 

\noindent \textbf{Baselines.} We explore the trainable calibration methods including cross-entropy(XE),  KDE-XE, focal, inverse focal, AURC, and dual focal loss. For focal loss, we consider FL-53 where $\gamma=5$ for $f_{y^\prime}(\boldsymbol{x}) \in [0, 0.2)$ and $\gamma=3$ otherwise.  For inverse focal loss, we note that when $\gamma>3.5$, it may not able to update the model properly due to the positive gradients, as discussed in Appendix~\ref{gradient:invfocal}. For dual focal loss, the implementation\footnote{\url{https://github.com/Linwei94/ICML2023-DualFocalLoss}} acts in the opposite direction of the intended principle but follows the concept of selective classification, as discussed in Appendix~\ref{gradient:dualfocal}. We additionally evaluate post-hoc calibration on models trained with the XE loss, considering both accuracy-preserving methods i.e. TS and methods that may alter accuracy, including MetaAcc, MetaMis\footnote{\url{https://github.com/maxc01/metacal}}~\citep{Ma2021a} in Appendix~\ref{appdix:Addresult}.

\noindent \textbf{Model selection.} For inverse focal and dual focal losses, the $\gamma$ parameter is fine-tuned on a split validation set following~\cite{tao2023dual}. For the $\mathbf{\text{AU}_{\mathbb{P}_{\mathcal{X}_\tau}}}$ loss in Eq.~\eqref{eq:auloss}, the smooth parameter $\mu$ and choice of CSF are selected using the same strategy. See Appendix~\ref{appendix:modelselection} for details.

\noindent \textbf{Influence of the CSFs.} We evaluate the $\mathbf{\text{AU}_{\mathbb{P}_{\mathcal{X}_\tau}}}$ loss across different CSFs, including MSP, Negative Entropy, and Softmax Margin. As illustrated in Figure~\ref{fig:influenceCSF}, the empirical calibration performance of $\mathbf{\text{AU}_{\mathbb{P}_{\mathcal{X}_\tau}}}$ exhibits noticeable sensitivity to the choice of CSF in most scenarios. For instance, MSP achieves the best $\widehat{\text{ECE}}$ on CIFAR-100 except for the WRN model. On Tiny-ImageNet, Negative Entropy gives the best performance for the ViT-Small, whereas MSP yields superior results for Swin-Tiny model.

\noindent \textbf{Results for trainable calibration methods.} Table~\ref{tab:combined_acc_ece} presents the test performance in terms of $\widehat{\text{ECE}}$ across various loss functions. Overall, the proposed $\mathbf{\text{AU}_{\mathbb{P}_{\mathcal{X}_\tau}}}$ loss demonstrates good calibration performance compared to baseline methods, while focal loss and inverse focal loss are more variable, depending on both the model architecture and dataset. As visualized in Fig.~\ref{fig:rc_ece}, focal loss tends to produce overconfident predictions on samples, whereas inverse focal loss and AURC exhibit underconfidence. In contrast, $\mathbf{\text{AU}_{\mathbb{P}_{\mathcal{X}_\tau}}}$ yields more balanced and better-calibrated confidence scores. We also observe that Dual Focal loss exhibits strong performance across all datasets and models; however, this may be attributed to its implementation, which operates in the opposite direction of the stated objective and resembles the principal of selective classification, as discussed in Appendix~\ref{appendix:gradient}. Moreover, results in Appendix Table~\ref{tab:cwece_ts_split} shows that the top-label calibration performance can be further improved via TS, highlighting the benefit of post-hoc methods. For non–accuracy-preserving post-hoc methods i.e. MetaAcc, MetaMis, they yield additional reductions in both $\widehat{\text{cwECE}}$ and $\widehat{\text{ECE}}^a$ (see Appendix Table~\ref{tab:ece_a_ts_split}) but often at the expense of predictive accuracy. 

Table~\ref{tab:cwece_ts_split} provides an evaluation of $\widehat{\text{cwECE}}$ and accuracy across different models and datasets. The $\text{AU}_{\mathbb{P}_{\mathcal{X}_\tau}}$ loss consistently demonstrates good calibration performance, achieving the best $\widehat{\text{cwECE}}$ in the majority of cases. Furthermore, its empirical classwise calibration performance can be further enhanced with TS on the CIFAR-10/100 datasets. In contrast, FL-53 performs poorly relative to other trainable methods on both CIFAR-10 and Tiny-ImageNet, indicating its limited effectiveness in class-wise calibration settings.

\section{Conclusion and Limitation}

In this work, we found a theoretical connection between selective classification and calibration error minimization, offering insights into why certain reweighting schemes effectively improve calibration. Building on this foundation, we propose a differentiable loss function derived from selective classification principles, which targets the selective risk in low-confidence regions where confidence ranking is meaningful. Our approach utilizes a bin-based CDF approximation, ensuring computational efficiency while maintaining differentiability for any choice of CSFs.

Interestingly, we find that the dual focal loss implementation implicitly follows similar principles by down-weighting hard examples, although this connection was not explicitly made in their work. In contrast, our method is directly grounded in principled theoretical foundations, contributing to its superior empirical calibration performance across diverse datasets and architectures. A limitation of our approach is its reliance on the quality of the chosen CSF. Our empirical evaluations demonstrate that CSF performance varies significantly across different datasets and model architectures. We therefore anticipate that incorporating more reliable confidence estimation will further enhance calibration performance.

\section*{Acknowledgement}
This research received funding from the Flemish Government (AI Research Program) and the Research Foundation - Flanders (FWO) through project number G0G2921N. HZ is supported by the China Scholarship Council. We acknowledge EuroHPC JU for awarding the project ID EHPC-BEN-2025B22-061 and EHPC-BEN-2025B22-037 access to the EuroHPC supercomputer LEONARDO, hosted by CINECA (Italy) and the LEONARDO consortium.

\bibliography{ref}

@article{mukhoti2020calibrating,
  title={Calibrating deep neural networks using focal loss},
  author={Mukhoti, Jishnu and Kulharia, Viveka and Sanyal, Amartya and Golodetz, Stuart and Torr, Philip and Dokania, Puneet},
  journal={Advances in neural information processing systems},
  volume={33},
  pages={15288--15299},
  year={2020}
}

@inproceedings{naeini2015obtaining,
  title={Obtaining well calibrated probabilities using bayesian binning},
  author={Naeini, Mahdi Pakdaman and Cooper, Gregory and Hauskrecht, Milos},
  booktitle={Proceedings of the AAAI conference on artificial intelligence},
  volume={29},
  year={2015}
}

@inproceedings{dosovitskiy2020image,
  title={An Image is Worth 16x16 Words: Transformers for Image Recognition at Scale},
  author={Dosovitskiy, Alexey and Beyer, Lucas and Kolesnikov, Alexander and Weissenborn, Dirk and Zhai, Xiaohua and Unterthiner, Thomas and Dehghani, Mostafa and Minderer, Matthias and Heigold, G and Gelly, S and others},
  booktitle={International Conference on Learning Representations},
  year={2020}
}

@inproceedings{xia2022moderate,
  title={Moderate coreset: A universal method of data selection for real-world data-efficient deep learning},
  author={Xia, Xiaobo and Liu, Jiale and Yu, Jun and Shen, Xu and Han, Bo and Liu, Tongliang},
  booktitle={The Eleventh International Conference on Learning Representations},
  year={2022}
}

@misc{yoshioka2024visiontransformers,
  author       = {Kentaro Yoshioka},
  title        = {vision-transformers-cifar10: Training Vision Transformers (ViT) and related models on CIFAR-10},
  year         = {2024},
  publisher    = {GitHub},
  howpublished = {\url{https://github.com/kentaroy47/vision-transformers-cifar10}}
}

@inproceedings{liu2021swin,
  title={Swin transformer: Hierarchical vision transformer using shifted windows},
  author={Liu, Ze and Lin, Yutong and Cao, Yue and Hu, Han and Wei, Yixuan and Zhang, Zheng and Lin, Stephen and Guo, Baining},
  booktitle={Proceedings of the IEEE/CVF international conference on computer vision},
  pages={10012--10022},
  year={2021}
}

@inproceedings{he2016identity,
  title={Identity mappings in deep residual networks},
  author={He, Kaiming and Zhang, Xiangyu and Ren, Shaoqing and Sun, Jian},
  booktitle={Computer Vision--ECCV 2016: 14th European Conference, Amsterdam, The Netherlands, October 11--14, 2016, Proceedings, Part IV 14},
  pages={630--645},
  year={2016},
  organization={Springer}
}

@article{liu2024deepseek,
  title={DeepSeek-V3 Technical Report},
  author={Liu, Aixin and Feng, Bei and Xue, Bing and Wang, Bingxuan and Wu, Bochao and Lu, Chengda and Zhao, Chenggang and Deng, Chengqi and Zhang, Chenyu and Ruan, Chong and others},
  journal={arXiv e-prints},
  pages={arXiv--2412},
  year={2024}
}

@article{achiam2023gpt,
  title={Gpt-4 technical report},
  author={Achiam, Josh and Adler, Steven and Agarwal, Sandhini and Ahmad, Lama and Akkaya, Ilge and Aleman, Florencia Leoni and Almeida, Diogo and Altenschmidt, Janko and Altman, Sam and Anadkat, Shyamal and others},
  journal={arXiv preprint arXiv:2303.08774},
  year={2023}
}

@article{haggenmuller2021skin,
  title={Skin cancer classification via convolutional neural networks: systematic review of studies involving human experts},
  author={Haggenm{\"u}ller, Sarah and Maron, Roman C and Hekler, Achim and Utikal, Jochen S and Barata, Catarina and Barnhill, Raymond L and Beltraminelli, Helmut and Berking, Carola and Betz-Stablein, Brigid and Blum, Andreas and others},
  journal={European Journal of Cancer},
  volume={156},
  pages={202--216},
  year={2021},
  publisher={Elsevier}
}

@inproceedings{gupta2022top,
  title={Top-label calibration and multiclass-to-binary reductions},
  author={Gupta, Chirag and Ramdas, Aaditya},
  booktitle={International Conference on Learning Representations},
  year={2022},
}

@inproceedings{panchenko2022class,
  title={Class-wise and reduced calibration methods},
  author={Panchenko, Michael and Benmerzoug, Anes and de Benito Delgado, Miguel},
  booktitle={2022 21st IEEE International Conference on Machine Learning and Applications (ICMLA)},
  pages={1093--1100},
  year={2022},
  organization={IEEE}
}

@article{kumar2019verified,
  title={Verified uncertainty calibration},
  author={Kumar, Ananya and Liang, Percy S and Ma, Tengyu},
  journal={Advances in neural information processing systems},
  volume={32},
  year={2019}
}

@inproceedings{charoenphakdee2021focal,
  title={On focal loss for class-posterior probability estimation: A theoretical perspective},
  author={Charoenphakdee, Nontawat and Vongkulbhisal, Jayakorn and Chairatanakul, Nuttapong and Sugiyama, Masashi},
  booktitle={Proceedings of the IEEE/CVF Conference on Computer Vision and Pattern Recognition},
  pages={5202--5211},
  year={2021}
}

@inproceedings{Ma2021a,
  TITLE = {Meta-Cal: Well-controlled Post-hoc Calibration by Ranking},
  AUTHOR = {Ma, Xingchen and Blaschko, Matthew B.},
  BOOKTITLE = {International Conference on Machine Learning},
  YEAR = {2021},
}

@inproceedings{guo2017calibration,
  title={On calibration of modern neural networks},
  author={Guo, Chuan and Pleiss, Geoff and Sun, Yu and Weinberger, Kilian Q},
  booktitle={International conference on machine learning},
  pages={1321--1330},
  year={2017},
  organization={PMLR}
}

@misc{rw2019timm,
  author = {Ross Wightman},
  title = {PyTorch Image Models},
  year = {2019},
  publisher = {GitHub},
  journal = {GitHub repository},
  doi = {10.5281/zenodo.4414861},
  howpublished = {\url{https://github.com/rwightman/pytorch-image-models}}
}

@article{popordanoska2022consistent,
  title={A consistent and differentiable lp canonical calibration error estimator},
  author={Popordanoska, Teodora and Sayer, Raphael and Blaschko, Matthew},
  journal={Advances in Neural Information Processing Systems},
  volume={35},
  pages={7933--7946},
  year={2022}
}

@inproceedings{zagoruyko2016wide,
  title={Wide Residual Networks},
  author={Zagoruyko, Sergey and Komodakis, Nikos},
  booktitle={Procedings of the British Machine Vision Conference 2016},
  pages={87--1},
  year={2016},
  organization={British Machine Vision Association}
}

@inproceedings{karandikar2021soft,
  title={Soft calibration objectives for neural networks},
  author={Karandikar, Abhay and Cain, Nicholas and Tran, Dinh and Lakshminarayanan, Balaji and Shlens, Jonathon and Mozer, Michael C and Roelofs, Rebecca},
  booktitle={Advances in Neural Information Processing Systems},
  volume={34},
  pages={29768--29779},
  year={2021}
}

@inproceedings{kumar2018trainable,
  title={Trainable calibration measures for neural networks from kernel mean embeddings},
  author={Kumar, Aviral and Sarawagi, Sunita and Jain, Ujjwal},
  booktitle={Proceedings of the 35th International Conference on Machine Learning},
  pages={2805--2814},
  year={2018},
  organization={PMLR}
}

@inproceedings{krishnan2020improving,
  title={Improving model calibration with accuracy versus uncertainty optimization},
  author={Krishnan, Ramesh and Tickoo, Omesh},
  booktitle={Advances in Neural Information Processing Systems},
  volume={33},
  pages={18237--18248},
  year={2020}
}

@article{bohdal2023meta,
  title={Meta-Calibration: Learning of Model Calibration Using Differentiable Expected Calibration Error},
  author={Bohdal, Ondrej and Yang, Yongxin and Hospedales, Timothy M},
  journal={Transactions on Machine Learning Research},
  pages={1--21},
  year={2023}
}

@inproceedings{hui2020evaluation,
  title={Evaluation of neural architectures trained with square loss vs cross-entropy in classification tasks},
  author={Hui, L},
  booktitle={The Ninth International Conference on Learning Representations (ICLR 2021)},
  year={2020}
}

@inproceedings{tao2023dual,
  title={Dual focal loss for calibration},
  author={Tao, Linwei and Dong, Minjing and Xu, Chang},
  booktitle={International Conference on Machine Learning},
  pages={33833--33849},
  year={2023},
  organization={PMLR}
}

@inproceedings{ding2021local,
  title={Local temperature scaling for probability calibration},
  author={Ding, Zhipeng and Han, Xu and Liu, Peirong and Niethammer, Marc},
  booktitle={Proceedings of the IEEE/CVF International Conference on Computer Vision},
  pages={6889--6899},
  year={2021}
}

@article{ghosh2022adafocal,
  title={Adafocal: Calibration-aware adaptive focal loss},
  author={Ghosh, Arindam and Schaaf, Thomas and Gormley, Matthew},
  journal={Advances in Neural Information Processing Systems},
  volume={35},
  pages={1583--1595},
  year={2022}
}

@inproceedings{he2016deep,
  title={Deep residual learning for image recognition},
  author={He, Kaiming and Zhang, Xiangyu and Ren, Shaoqing and Sun, Jian},
  booktitle={Proceedings of the IEEE conference on computer vision and pattern recognition},
  pages={770--778},
  year={2016}
}

@inproceedings{deng2009imagenet,
  title={Imagenet: A large-scale hierarchical image database},
  author={Deng, Jia and Dong, Wei and Socher, Richard and Li, Li-Jia and Li, Kai and Fei-Fei, Li},
  booktitle={2009 IEEE conference on computer vision and pattern recognition},
  pages={248--255},
  year={2009},
  organization={Ieee}
}

@inproceedings{zhou2024novelcharacterizationpopulationarea,
      title={A Novel Characterization of the Population Area Under the 
Risk Coverage Curve (AURC) and Rates of Finite Sample Estimators}, 
      author={Han Zhou and Jordy Van Landeghem and Teodora Popordanoska 
and Matthew B. Blaschko},
      booktitle={International Conference on Machine Learning},
      year={2025},
      eprint={2410.15361},
      archivePrefix={arXiv},
      primaryClass={stat.ML},
}

@article{belghazi2021classifiers,
  title={What classifiers know what they don't?},
  author={Belghazi, Mohamed Ishmael and Lopez-Paz, David},
  journal={arXiv preprint arXiv:2107.06217},
  year={2021}
}

@inproceedings{ding2020revisiting,
  title={Revisiting the evaluation of uncertainty estimation and its application to explore model complexity-uncertainty trade-off},
  author={Ding, Yukun and Liu, Jinglan and Xiong, Jinjun and Shi, Yiyu},
  booktitle={Proceedings of the IEEE/CVF Conference on Computer Vision and Pattern Recognition Workshops},
  pages={4--5},
  year={2020}
}

@article{krizhevsky2009learning,
  title={Learning Multiple Layers of Features from Tiny Images},
  author={Krizhevsky, A},
  journal={Master's thesis, University of Tront},
  year={2009}
}

@inproceedings{geifman2018bias,
  title={Bias-reduced uncertainty estimation for deep neural classifiers},
  author={Geifman, Yonatan and Uziel, Guy and El-Yaniv, Ran},
  booktitle={International Conference on Learning Representations},
  year={2019}
}

@inproceedings{hendrycks2016baseline,
  title={A Baseline for Detecting Misclassified and Out-of-Distribution Examples in Neural Networks},
  author={Hendrycks, Dan and Gimpel, Kevin},
  booktitle={International Conference on Learning Representations},
  year={2022}
}

@article{liu2020energy,
  title={Energy-based out-of-distribution detection},
  author={Liu, Weitang and Wang, Xiaoyun and Owens, John and Li, Yixuan},
  journal={Advances in neural information processing systems},
  volume={33},
  pages={21464--21475},
  year={2020}
}

@inproceedings{esaki2024accuracy,
  title={Accuracy-Preserving Calibration via Statistical Modeling on Probability Simplex},
  author={Esaki, Yasushi and Nakamura, Akihiro and Kawano, Keisuke and Tokuhisa, Ryoko and Kutsuna, Takuro},
  booktitle={International Conference on Artificial Intelligence and Statistics},
  pages={1666--1674},
  year={2024},
  organization={PMLR}
}

@inproceedings{lin2017focal,
  title={Focal loss for dense object detection},
  author={Lin, Tsung-Yi and Goyal, Priya and Girshick, Ross and He, Kaiming and Doll{\'a}r, Piotr},
  booktitle={Proceedings of the IEEE international conference on computer vision},
  pages={2980--2988},
  year={2017}
}

@article{brier1950verification,
  title={Verification of forecasts expressed in terms of probability},
  author={Brier, Glenn W},
  journal={Monthly weather review},
  volume={78},
  number={1},
  pages={1--3},
  year={1950}
}

@article{geifman2017selective,
  title={Selective classification for deep neural networks},
  author={Geifman, Yonatan and El-Yaniv, Ran},
  journal={Advances in neural information processing systems},
  volume={30},
  year={2017}
}

@inproceedings{galilcan,
  title={What Can we Learn From The Selective Prediction And Uncertainty Estimation Performance Of 523 Imagenet Classifiers?},
  author={Galil, Ido and Dabbah, Mohammed and El-Yaniv, Ran},
  booktitle={The Eleventh International Conference on Learning Representations},
year={2023}
}

@article{liu2024optimizing,
  title={Optimizing calibration by gaining aware of prediction correctness},
  author={Liu, Yuchi and Wang, Lei and Zou, Yuli and Zou, James and Zheng, Liang},
  journal={arXiv preprint arXiv:2404.13016},
  year={2024}
}

@article{platt1999probabilistic,
  title={Probabilistic outputs for support vector machines and comparisons to regularized likelihood methods},
  author={Platt, John and others},
  journal={Advances in large margin classifiers},
  volume={10},
  number={3},
  pages={61--74},
  year={1999},
  publisher={Cambridge, MA}
}

@article{kull2019beyond,
  title={Beyond temperature scaling: Obtaining well-calibrated multi-class probabilities with dirichlet calibration},
  author={Kull, Meelis and Perello Nieto, Miquel and K{\"a}ngsepp, Markus and Silva Filho, Telmo and Song, Hao and Flach, Peter},
  journal={Advances in neural information processing systems},
  volume={32},
  year={2019}
}

@inproceedings{kull2017beta,
  title={Beta calibration: a well-founded and easily implemented improvement on logistic calibration for binary classifiers},
  author={Kull, Meelis and Silva Filho, Telmo and Flach, Peter},
  booktitle={Artificial intelligence and statistics},
  pages={623--631},
  year={2017},
  organization={PMLR}
}

@inproceedings{wenger2020non,
  title={Non-parametric calibration for classification},
  author={Wenger, Jonathan and Kjellstr{\"o}m, Hedvig and Triebel, Rudolph},
  booktitle={International Conference on Artificial Intelligence and Statistics},
  pages={178--190},
  year={2020},
  organization={PMLR}
}

@article{wang2021rethinking,
  title={Rethinking calibration of deep neural networks: Do not be afraid of overconfidence},
  author={Wang, Deng-Bao and Feng, Lei and Zhang, Min-Ling},
  journal={Advances in Neural Information Processing Systems},
  volume={34},
  pages={11809--11820},
  year={2021}
}

@incollection{komisarenko2024improving,
  title={Improving Calibration by Relating Focal Loss, Temperature Scaling, and Properness},
  author={Komisarenko, Viacheslav and Kull, Meelis},
  booktitle={ECAI 2024},
  pages={1535--1542},
  year={2024},
  publisher={IOS Press}
}
\bibliographystyle{icml2026}

\newpage
\appendix
\onecolumn

\section{Definitions} \label{apdix:def}
\setcounter{figure}{0}
\setcounter{table}{0}

\begin{definition}[Brier Score (BS)]
   The Brier Score of a model $f: \mathcal{X} \rightarrow \Delta^{k-1}$ is defined as:
   \begin{align}
      \text{BS}(f) = \mathbb{E} \left[ \| f(\boldsymbol{x}) - \boldsymbol{y} \|_2^2 \right] 
   \end{align}
   where $f(\boldsymbol{x})$ is the predicted probabilities and $\boldsymbol{y}$ is the one-hot encoded label.
\end{definition}

\noindent\textbf{CSFs.} The CSFs are generally defined as functions of the predicted probabilities $\boldsymbol{p}$, which are the outputs by passing the logits $\boldsymbol{z}$ produced by the model for the input $\boldsymbol{x}$ through the softmax function $\sigma(\cdot)$, expressed as $\boldsymbol{p}= \sigma(\boldsymbol{z}) \in \mathbb{R}^k$. The specific forms of these CSFs are outlined as follows:
\begin{table}[H] 
\caption{Commonly Used CSFs} 
\label{tb:csfs}
\centering
\begin{tabular}{|>{\raggedright\arraybackslash}m{3.5cm}|>{\raggedright\arraybackslash}m{3.5cm}|}
% \begin{tabular}{|c|}
\hline
\textbf{Method} & \textbf{Equation} \\ \hline
MSP & $g(\boldsymbol{p})= \max_{i=1}^k p_i$ \\ \hline
% MaxLogit & $g(\mathbf{z})=\max_{i=1}^K \mathbf{z}_i$ \\ \hline
Softmax Margin& \begin{tabular}[c]{@{}l@{}}$g(\boldsymbol{p})=p_i-\max_{j\neq i} p_j$  \\ with $i=\arg\max_{i=1}^k p_i$\end{tabular}  \\ \hline
Negative Entropy & $g(\boldsymbol{p}) = \sum_{i=1}^k \boldsymbol{z}_i \log \boldsymbol{z}
_i$ \\  \hline
\end{tabular}
\end{table}

\section{MISSING PROOFS}\label{apdix:proof}

\subsection{Proof of Prop.~\ref{prop:beta}} \label{proof:beta}
\begin{proof}
We begin with the expression for \( \beta(\boldsymbol{x}) \) from Eq.~\eqref{eq:defalpha}:
\begin{align}
\beta(\boldsymbol{x})=\mathbb{E}_{\tilde{\boldsymbol{x}}\sim \mathbb{P}_{\mathcal{X}_\tau}} \left(\frac{\mathbb{I}\left[g(\boldsymbol{x}) \ge g(\tilde{\boldsymbol{x}})\right] }{ \mathbb{E}_{\boldsymbol{x}^\prime \sim \mathbb{P}(\boldsymbol{x} )} \mathbb{I}\left[g(\boldsymbol{x}^\prime) \ge g(\tilde{\boldsymbol{x}})\right]} \right)  = \mathbb{E}_{\tilde{\boldsymbol{x}} \sim \mathbb{P}_{\mathcal{X}_\tau}} \left[ \frac{\mathbb{I}[g(\boldsymbol{x}) \geq g(\tilde{\boldsymbol{x}})]}{1 - G(g(\tilde{\boldsymbol{x}}))} \right],
\end{align}
where \( \mathbb{I}[g(\boldsymbol{x}) \geq g(\tilde{\boldsymbol{x}})] \) is the indicator function, which is 1 when \( g(\boldsymbol{x}) \geq g(\tilde{\boldsymbol{x}}) \), and 0 otherwise. The expectation is taken over the conditional distribution \( \mathbb{P}_{\mathcal{X}_\tau} \), where \( g(\tilde{\boldsymbol{x}}) \leq \tau \).

Next, we evaluate this expression by considering two different conditions for \( \boldsymbol{x} \):

\begin{itemize}
    \item If \(\boldsymbol{x} \in \mathcal{X} \setminus \mathcal{X}_\tau\):

    In this case, the indicator function \( \mathbb{I}[g(\boldsymbol{x}) \geq g(\tilde{\boldsymbol{x}})] \) is always 1 and we integrate over the region \( \mathcal{X}_\tau \), where \( g(\tilde{\boldsymbol{x}}) \leq \tau \). Thus, the expression for \( \beta(\boldsymbol{x}) \) becomes:
    \begin{align}
    \beta(\boldsymbol{x})  = \int_{\mathcal{X}_\tau} \left( \frac{\mathbb{I}[\tau \ge g(\tilde{\boldsymbol{x}})]}{1 - G(g(\tilde{\boldsymbol{x}}))} \right) \frac{d\mathbb{P}(\tilde{\boldsymbol{x}})}{\Pr(\mathcal{X}_\tau)}= \frac{1}{G(\tau)}
    \int_{g(\tilde{\boldsymbol{x}}) \leq \tau} \frac{1}{1 - G(g(\tilde{\boldsymbol{x}}))} d\mathbb{P}(\tilde{\boldsymbol{x}}). 
    \end{align}
    where $\Pr(\mathcal{X}_\tau) = G(\tau)$. Since \( G(g(\tilde{\boldsymbol{x}})) \) is the CDF of \( g(\tilde{\boldsymbol{x}}) \), which is monotonically increasing in \( g(\tilde{\boldsymbol{x}}) \), we can change the variable of integration to \( G(g(\tilde{\boldsymbol{x}})) \). The CDF \( G(\tilde{\boldsymbol{x}}) \) takes values between 0 and \( G(\tau) \) as \( g(\tilde{\boldsymbol{x}}) \) ranges from 0 to \( \tau \). Thus, we rewrite the integral as:
    \begin{align}
    \beta(\boldsymbol{x}) = \frac{1}{G(\tau)} \int_0^{G(\tau)} \frac{1}{1 - t} dt,
  \end{align}
    where \( t = G(g(\tilde{\boldsymbol{x}})) \). The integral of \( \frac{1}{1 - t} \) is straightforward to compute:
     \begin{align}
    \int_0^{G(\tau)} \frac{1}{1 - t} dt = -\ln \left(1 - G\left(\tau\right)\right).
   \end{align}
    Therefore, when \( g(\boldsymbol{x}) \geq \tau \), we obtain:
    \begin{align}
    \beta(\boldsymbol{x}) = -\frac{1}{G(\tau)} \ln(1 -G(\tau).
   \end{align}

    \item If \( \boldsymbol{x}  \in  \mathcal{X}_\tau  \):

    In this case, the indicator function \( \mathbb{I}[g(\boldsymbol{x}) \geq g(\tilde{\boldsymbol{x}})] \) is nonzero only when \( g(\boldsymbol{x}) \geq g(\tilde{\boldsymbol{x}}) \). Therefore the expression for \( \beta(\boldsymbol{x}) \) becomes:
  \begin{align}
    \beta(x) =\frac{1}{G(\tau)} \int_{g(\tilde{\boldsymbol{x}}) \leq g(\boldsymbol{x})} \frac{1}{1 - G(g(\tilde{\boldsymbol{x}}))} d\mathbb{P}(\tilde{\boldsymbol{x}})).
 \end{align}
    Similar to the first case, this results in:
  \begin{align}
    \beta(\boldsymbol{x}) = -\frac{1}{G(\tau)} \ln \bigl(1 - G(\boldsymbol{x}) \bigr).
 \end{align}
\end{itemize}
Combining the results from both conditions, we conclude that \( \beta(\boldsymbol{x}) \) is given by:
\begin{align}
\beta(\boldsymbol{x}) = 
\begin{cases}
    -\frac{1}{G(\tau)} \ln(1 - G(\tau)) & \text{if } \boldsymbol{x}  \in \mathcal{X} \setminus \mathcal{X}_\tau , \\
    -\frac{1}{G(\tau)} \ln(1 - G(g(\boldsymbol{x}))) & \text{if } \boldsymbol{x}   \in  \mathcal{X}_\tau .
\end{cases}
 \end{align}
Since $G(\tau)$ is a constant determined by the distribution and the threshold, it acts as a global scalar for the total loss. In the context of gradient-based optimization, this factor can be absorbed into the learning rate without affecting the relative sample importance or the optimization trajectory. Thus, we utilize the simplified form in Eq.~\eqref{eq:equivalentbeta}.
\end{proof}

\section{Gradient analysis} \label{appendix:gradient}

\subsection{Focal loss}~\label{gradient:focal}
 We consider the gradient of focal loss with respect to the true class predicted probability \( p_{y^{\prime}} \). By applying the product rule, we obtain:
\begin{align}
 \nabla_{p_{y^{\prime}}} \left[ - (1 - p_{y^{\prime}})^\gamma \log(p_{y^{\prime}}) \right] 
  & =   \nabla_{p_{y^{\prime}}} \left[ - (1 - p_{y^{\prime}})^\gamma \right]  \log(p_{y^{\prime}}) + (1 - p_{y^{\prime}})^\gamma  \nabla_{p_{y^{\prime}}} \log(p_{y^{\prime}})  \nonumber\\
  & = \gamma (1 - p_{y^{\prime}})^{\gamma - 1} \log(p_{y^{\prime}}) - (1 - p_{y^{\prime}})^\gamma  \frac{1}{p_{y^{\prime}}}. \nonumber
\end{align}
We can see that its gradient does shrink to $0$ when $p_{y^{\prime}} \rightarrow 1$ (see Fig.~\ref{fig:focalgrad}).

\begin{figure}[htbp]
% \footnotesize
\centering
\begin{minipage}[t]{0.49\linewidth}
    \centering
    \includegraphics[width=0.7\linewidth]{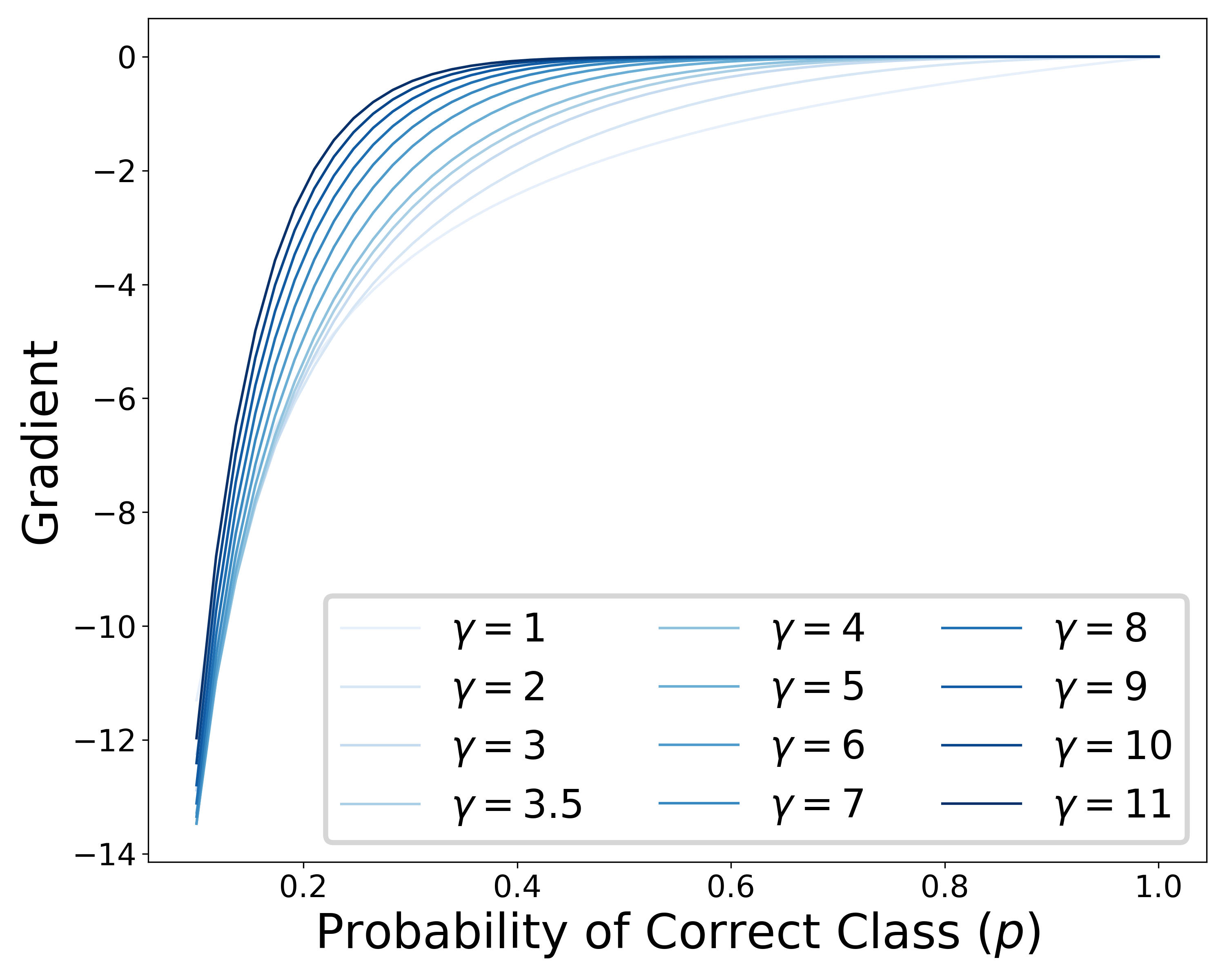}
    \subcaption{Focal loss}
    \label{fig:focalgrad}
\end{minipage}%
\begin{minipage}[t]{0.49\linewidth}
    \centering
    \includegraphics[width=0.7\linewidth]{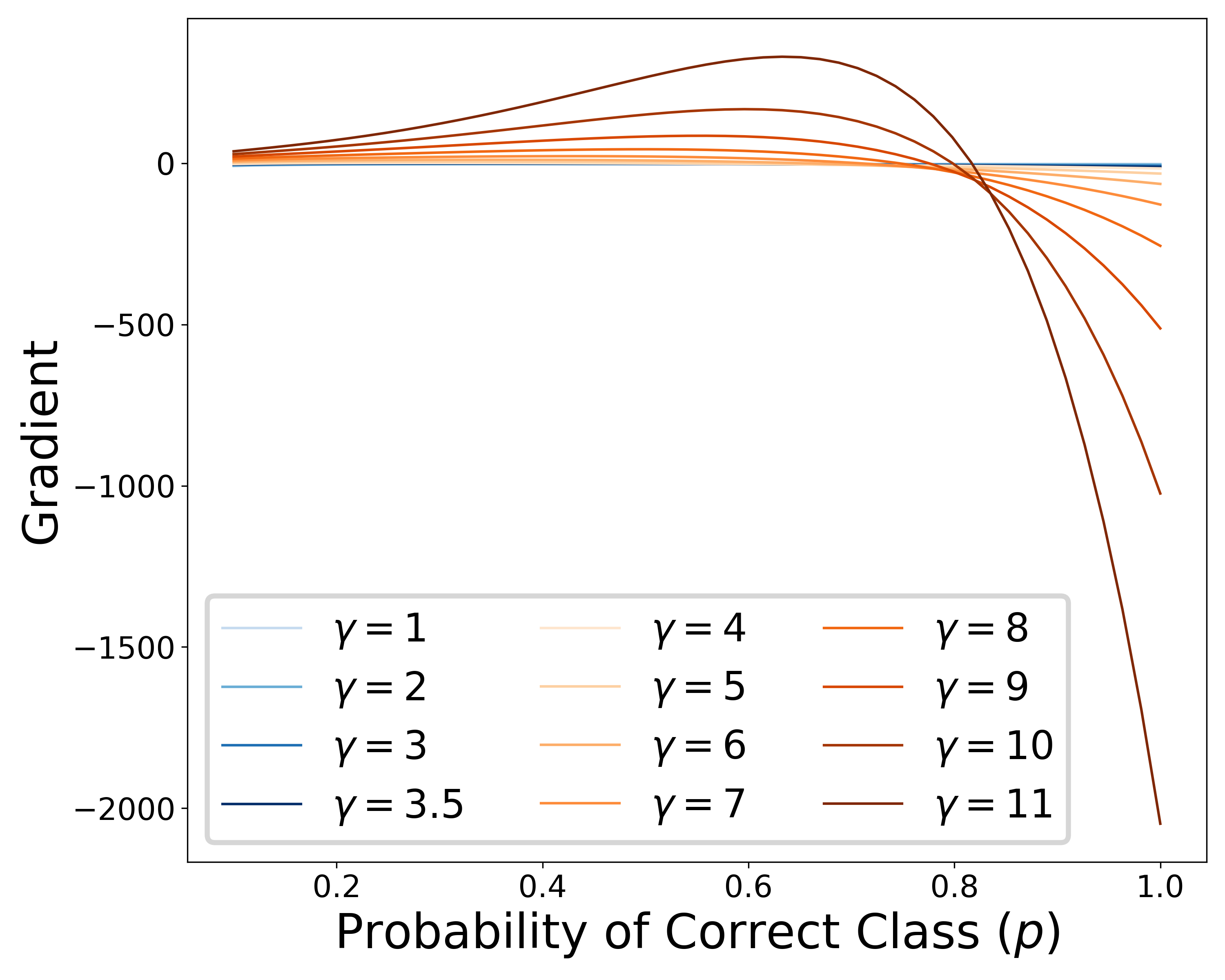}
    \subcaption{Inverse focal}
    \label{fig:invfocalgrad}
\end{minipage}
\vspace{-2mm}
\caption{The gradient of loss \(\ell\) with respect to \(p_{y^{\prime}}\).}
\label{fig:gradientcomparison}
\end{figure}

\subsection{Inverse Focal loss}~\label{gradient:invfocal}
For inverse focal loss, similarly we can compute its gradient with respect to \( p_{y^{\prime}} \):
\begin{align} 
 \nabla_{p_{y^{\prime}}} \left[ - (1 + p_{y^{\prime}})^\gamma \log(p_{y^{\prime}}) \right] 
& = \nabla_{p_{y^{\prime}}} \left[ - (1 + p_{y^{\prime}})^\gamma \right]  \log(p_{y^{\prime}}) + (1 + p_{y^{\prime}})^\gamma \nabla_{p_{y^{\prime}}} \log(p_{y^{\prime}}) \ \nonumber \\ & = \gamma (1 + p_{y^{\prime}})^{\gamma - 1} \log(p_{y^{\prime}}) - (1 + p_{y^{\prime}})^\gamma \frac{1}{p_{y^{\prime}}} \nonumber 
\end{align}
where we can see that its gradient does not shrink to $0$ when $p_{y^{\prime}} \rightarrow 1$ (see Fig.~\ref{fig:invfocalgrad}). However, we notice that when $\gamma$ is greater than $4$, the gradient can be positive in some regions. In these cases, the inverse focal loss may not be suitable to optimize the model.

\subsection{Dual Focal loss}~\label{gradient:dualfocal}
Dual Focal loss (DFL)~\cite{tao2023dual} generalizes Focal loss by incorporating both $p_{y'}$ (the probability of the ground-truth class) and the probability of the most competitive alternative class, $p_j$. The intended formulation is
\begin{equation}
\label{eq:dfl}
\ell_{\text{DFL}}(\boldsymbol{p}, \boldsymbol{y}) 
\;=\; - \bigl(1 - p_{y'} + p_j \bigr)^{\gamma} \, \log p_{y'}.
\end{equation}
Here, the weighting factor depends on the margin $p_{y'} - p_j$, which quantifies the model’s uncertainty: samples with smaller margins (i.e., harder examples) are assigned larger weights and consequently incur higher penalties.  
\noindent In practice, however, the publicly released implementation diverges from Eq.~\eqref{eq:dfl}. Specifically, the loss is computed as
\begin{align}
\ell_{\text{impl}} \big|_{p_{y'} \ge p_j}
&= - \bigl(1 - p_{y'} + p_j \bigr)^\gamma \, \log p_{y'}, \\[4pt]
\ell_{\text{impl}} \big|_{p_{y'} < p_j}
&= - \bigl(1 - p_{y'} \bigr)^\gamma \, \log p_{y'},
\end{align}
which introduces a discontinuity at the tie point $p_{y'} = p_j$. As $p_{y'}$ falls below $p_j$, the competing probability is abruptly discarded and the formulation degenerates into the standard focal weighting, as shown in Fig.~\ref{fig:randomdataDFL}. This discontinuity results in the under-penalization of misclassified samples. More importantly, the realized variant acts in the opposite direction of the intended principle: by nullifying the strongest competitor precisely when $p_{y'} < p_j$, it reduces the penalty for samples with large uncertainty. This contradiction resembles the effect observed in selective classification, where hard samples are down-weighted rather than emphasized.
\begin{figure}[htbp]
    \centering
    \begin{subfigure}{0.49\linewidth}
        \centering
        \includegraphics[width=2.1in]{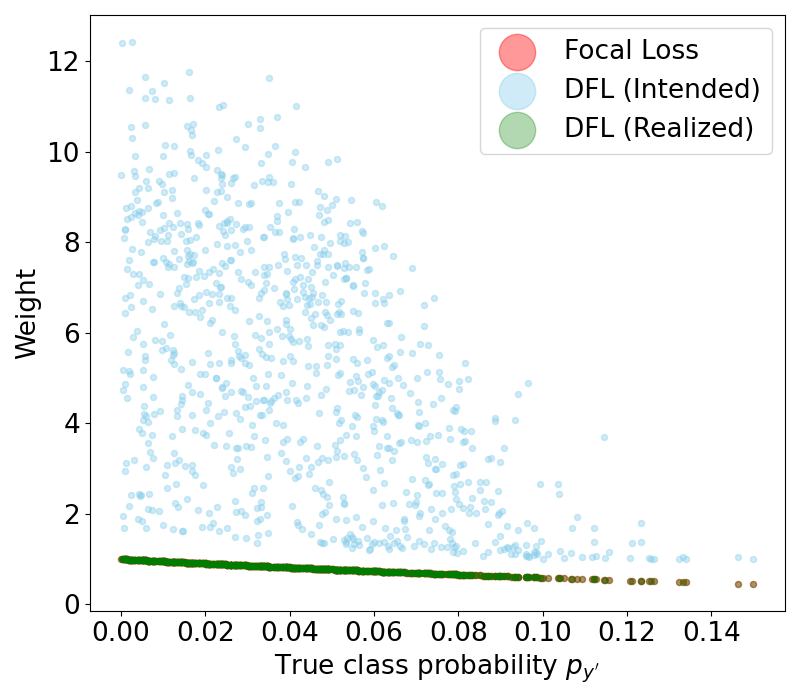}
        \caption{Weight}
        \label{fig:misweight}
    \end{subfigure}
    \hfill
    \begin{subfigure}{0.49\linewidth}
        \centering
        \includegraphics[width=2.1in]{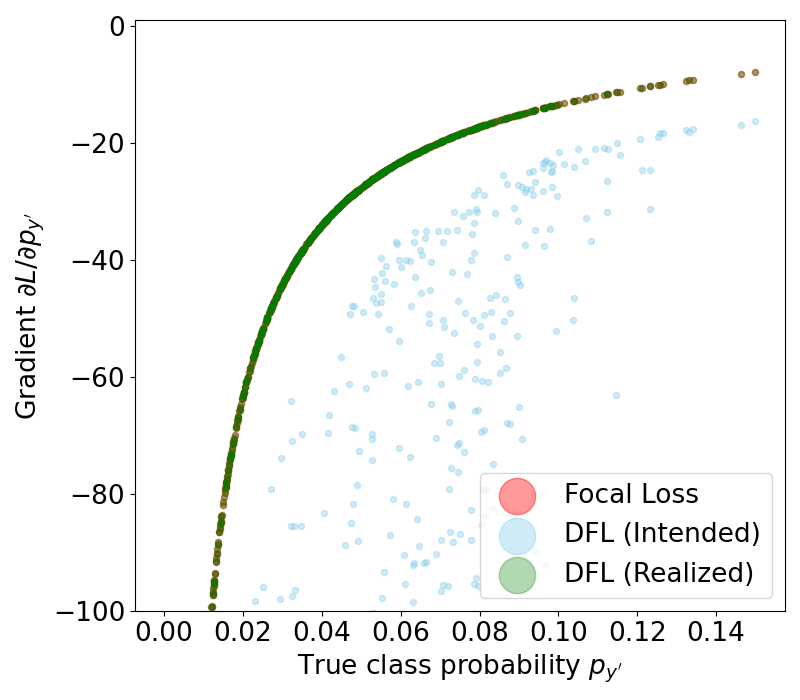}
        \caption{Gradient}
        \label{fig:misgradient}
    \end{subfigure}
 \caption{Illustration of DFL behavior on \textbf{misclassified} samples where $p_{y'} < p_j$.}
    \label{fig:randomdataDFL}
\end{figure}
\begin{figure*}[htbp]
\centering
\begin{minipage}{0.32\linewidth}
    \centering
    \includegraphics[width=1.0\linewidth]{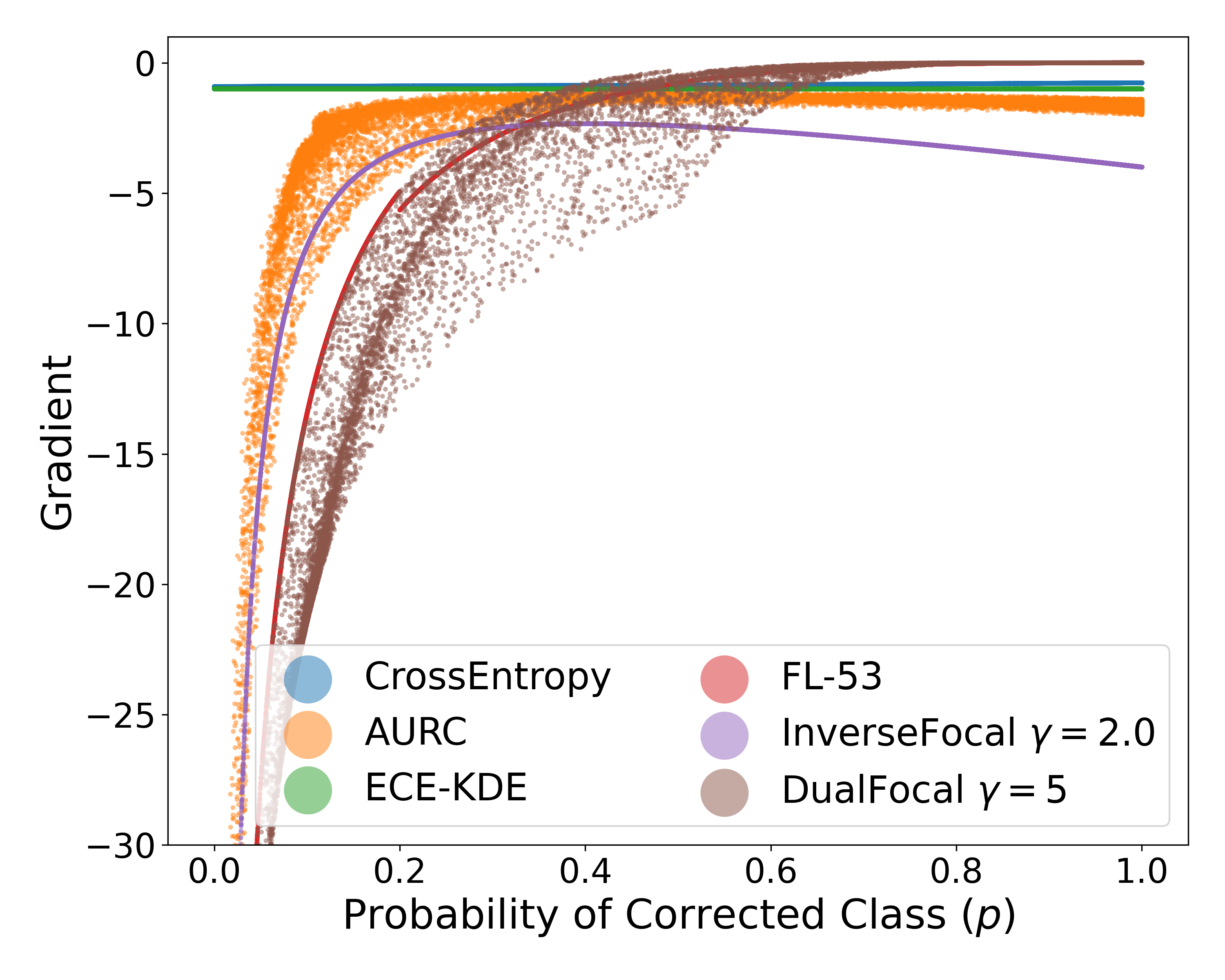}
    \subcaption{Epoch 40}
    \label{fig:gradientepoch40}
\end{minipage}
\begin{minipage}{0.32\linewidth}
    \centering
    \includegraphics[width=1.0\linewidth]{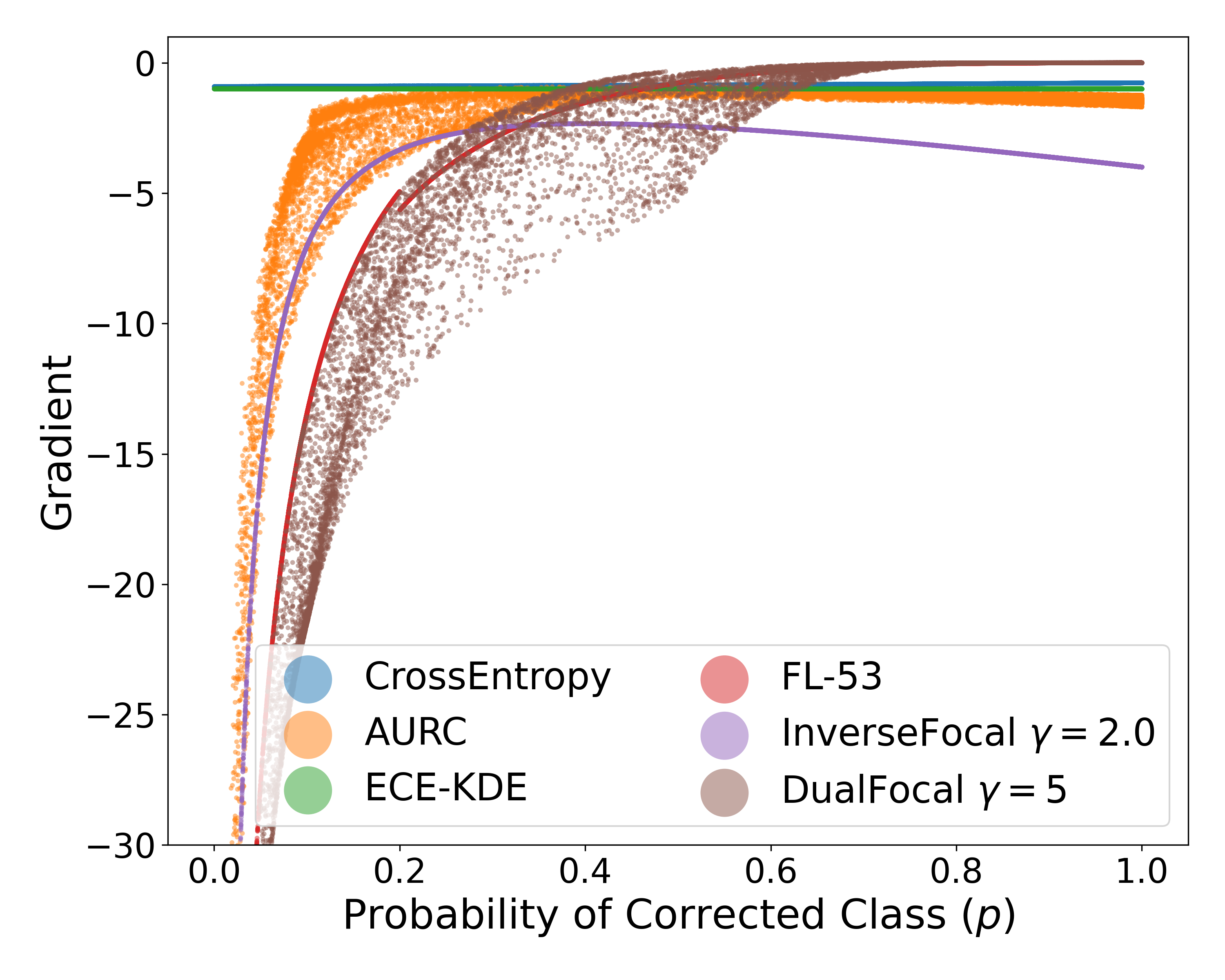}
    \subcaption{Epoch 120}
    \label{fig:gradientepoch120}
\end{minipage}
\begin{minipage}{0.32\linewidth}
    \centering
    \includegraphics[width=1.0\linewidth]{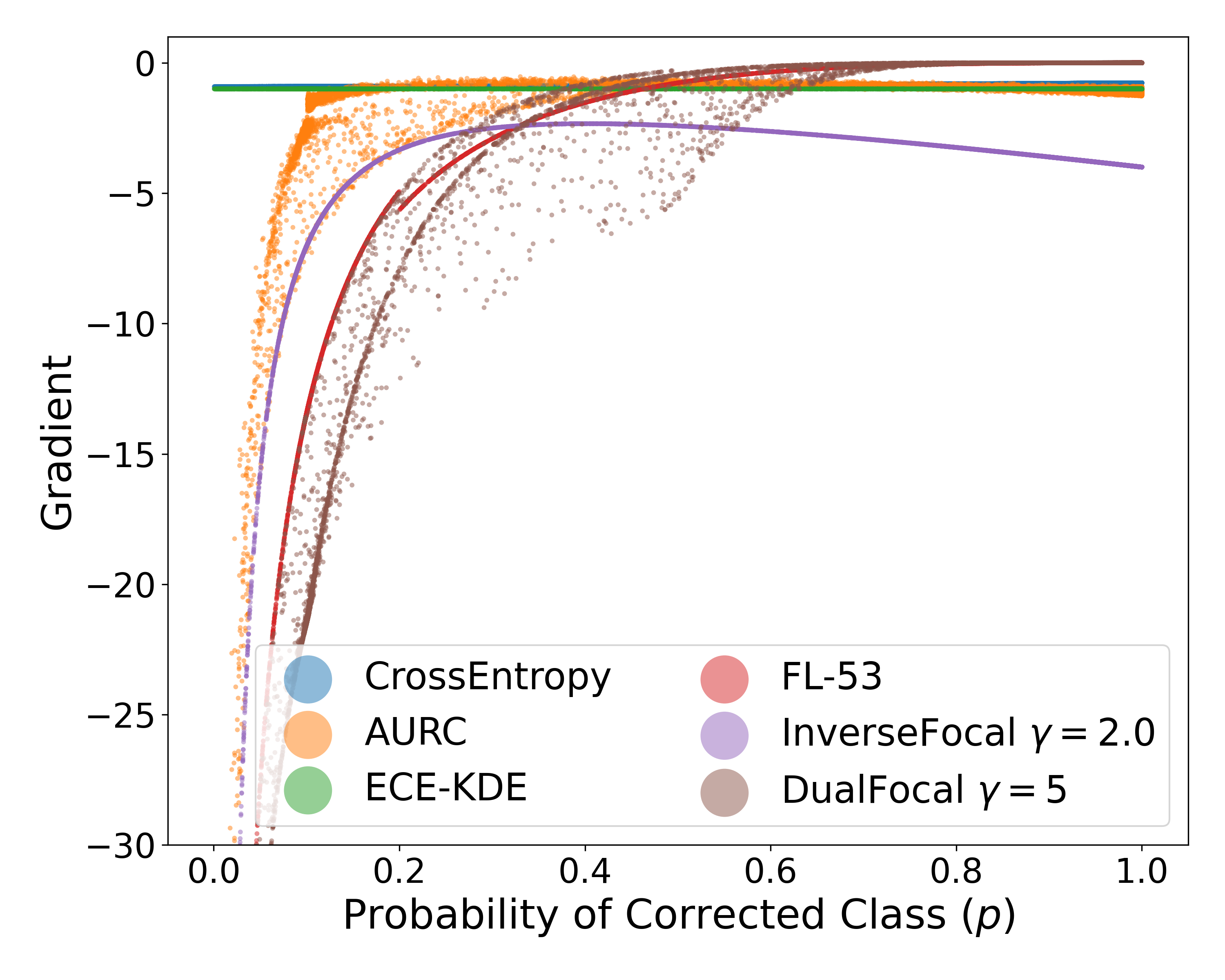}
    \subcaption{Epoch 200}
    \label{fig:gradientepoch200}
\end{minipage}

\caption{Evolution of the loss gradient $\partial \ell / \partial p_{y^{\prime}}$ with respect to model prediction probability $p_{y^{\prime}}$ for the same ResNet model evaluated on 20,000 CIFAR-10 training samples throughout training.}
\label{fig:gradients_over_epochs}
\end{figure*}
% Dual Focal loss (DFL)~\cite{tao2023dual} generalizes Focal loss by incorporating both $p_{y'}$ (the probability of the ground-truth class) and the probability of the most competitive alternative class, $p_j$. The intended formulation is
% \begin{equation}
% \label{eq:dfl}
% \ell_{\text{DFL}}(\boldsymbol{p}, \boldsymbol{y}) 
% \;=\; - \bigl(1 - p_{y'} + p_j \bigr)^{\gamma} \, \log p_{y'}.
% \end{equation}

\subsection{$\mathbf{\text{AU}_{\mathbb{P}_{\mathcal{X}_\tau}}}$ Loss} \label{gradient:au}
For the AU loss, we can compute the gradient with respect to the model output $\boldsymbol{p}_i = f_\theta(\boldsymbol{x}_i)$ as follows:
\begin{align} \label{eq:gradientAU}
\nabla_{\boldsymbol{p}_i}  \left( \widehat{\beta}_i  \ell(\boldsymbol{p}_i, \boldsymbol{y}_i) \right)
&= (\nabla_{\boldsymbol{p}_i} \widehat{\beta}_i) \, \ell(\boldsymbol{p}_i, \boldsymbol{y}_i) + \widehat{\beta}_i \, \nabla_{\boldsymbol{p}_i} \ell(\boldsymbol{p}_i, \boldsymbol{y}_i) ,
\end{align}
where the gradient of the sample weight $\widehat{\beta}_i$ is
\begin{align}
\nabla_{\boldsymbol{p}_i} \widehat{\beta}_i &=
\begin{cases}
\dfrac{1}{1 - \widehat{G}_\nu(s_i)} \, \nabla_{\boldsymbol{p}_i} \widehat{G}_\nu(s_i), & s_i < \tau,\\[1ex]
0, & s_i\ge \tau,
\end{cases} \nonumber
\end{align}
and the gradient of the smoothed bin-based CDF is
\begin{align}
\nabla_{\boldsymbol{p}_i} \widehat{G}_\nu(s_i) 
&= \frac{1}{n} \sum_{m=0}^{M-1} h_m \, \sigma'\Big(\frac{s_i - c_m}{\nu}\Big) \frac{1}{\nu} \, \nabla_{\boldsymbol{p}_i} s_i, \nonumber \\
\sigma'(t) &= \sigma(t)(1-\sigma(t)). \nonumber
\end{align}
We can see that the gradient decomposes into two components: one corresponding to the weight $\widehat{\beta}_i$ (which depends on the smoothed CDF) and the other corresponding to the original loss $\ell$. For high-confidence samples ($g(\boldsymbol{x}_i) \ge \tau$), the gradient reduces to the usual weighted loss term since $\nabla_{\boldsymbol{p}_i} \widehat{\beta}_i = 0$.
\section{ADDITIONAL EXPERIMENTS} 
\subsection{Experimental Setup} \label{appendix:trainingdetails}
The experiments are conducted using NVIDIA A100 GPUs. All models are trained using stochastic gradient descent (SGD) with a momentum of 0.9. For CIFAR-10 and CIFAR-100, we use a mini-batch size of 128, an initial learning rate of 0.1, and an $\ell_2$ weight decay of 0.0005. The learning rate is decayed by a factor of 0.2 at epochs 60, 120, and 160, and the training is conducted for 200 epochs for CIFAR-10 and 250 epochs for CIFAR-100. For all transformer-based models trained on Tiny-ImageNet, we use a mini-batch size of 256, an initial learning rate of $1 \times 10^{-4}$, a weight decay of 0.0005, and cosine learning rate scheduling for 90 epochs. For each loss function, we report the testing performance of the best model checkpoint selected based on the lowest validation loss after a warmup period (up to 40 epochs for CIFAR-10/100, 20 epochs for Tiny-ImageNet). Each configuration is run with three random seeds; we report the mean across runs.
% and the training is conducted for 200 epochs for CIFAR-10 and 250 epochs for CIFAR-100.

% \noindent \textbf{Post-hoc Calibration: } We utilized pretrained models that were trained with XE loss for all post-hoc experiments. For TS, the temperature is select by minimizing the ECE on the training set. For Meta-Cal, we follow the original configuration: for MetaMis (miscoverage-rate constraint), we set the miscoverage tolerance to 0.05; for MetaAcc (coverage-accuracy constraint), we target coverages of 0.97 on CIFAR-10, 0.87 on CIFAR-100, and 0.85 on Tiny-ImageNet. For Dirichlet Calibration, We employ 3-fold cross-validation with grid search over L2 regularization parameter $\lambda$, and we select optimal hyper-parameters by minimizing ECE on validation folds. For the $\mathrm{AU}_{P_{\mathcal{X}_\tau}}$ loss, we fine-tune each model with this loss for 5 epochs with an initial learning rate of $1\times10^{-4}$ on CIFAR-10/100 and $1\times10^{-5}$ on Tiny-ImageNet. 

\section{Model Selection} \label{appendix:modelselection}
\subsection{Baselines}
Following~\cite{mukhoti2020calibrating, tao2023dual}, we perform cross-validation to select the optimal $\gamma$ value. We hold out 5,000 images from the training set as a validation set, train on the rest, and choose the best $\gamma$ based on validation performance i.e. loss (see Table~\ref{tb:gammaforinvdual} for the selected $\gamma$ values). Fig.~\ref{fig:gammaselection} presents the performance on the validation set for Tiny-ImageNet data. For Inverse Focal Loss, Fig.~\ref{fig:gammaselection}a shows that $\gamma \ge 4$ consistently underperform, which should not be considered in practice. For KDE-XE, we use the regularization parameters recommended in the original work: $\lambda = 0.3$ for CIFAR-10 and $\lambda = 0.5$ for CIFAR-100.

\begin{figure}[htbp]
    \centering
    \begin{minipage}{0.9\linewidth}
        \centering
        \includegraphics[width=1\linewidth]{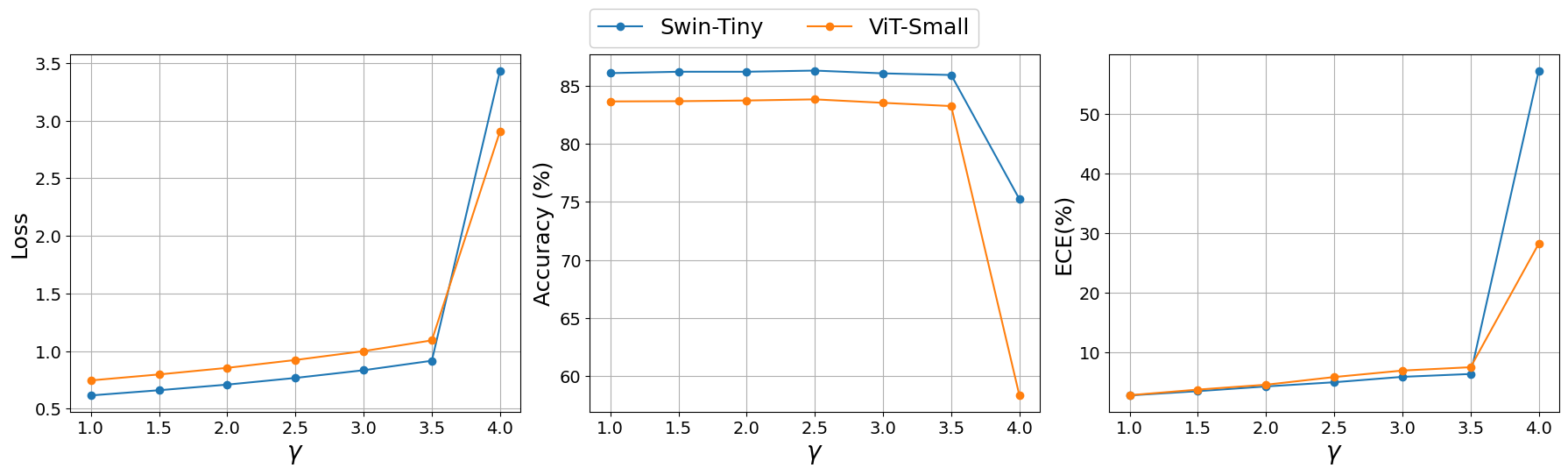}
        \subcaption{Inverse focal}
    \end{minipage}
    \begin{minipage}{0.9\linewidth}
        \centering
        \includegraphics[width=1\linewidth]{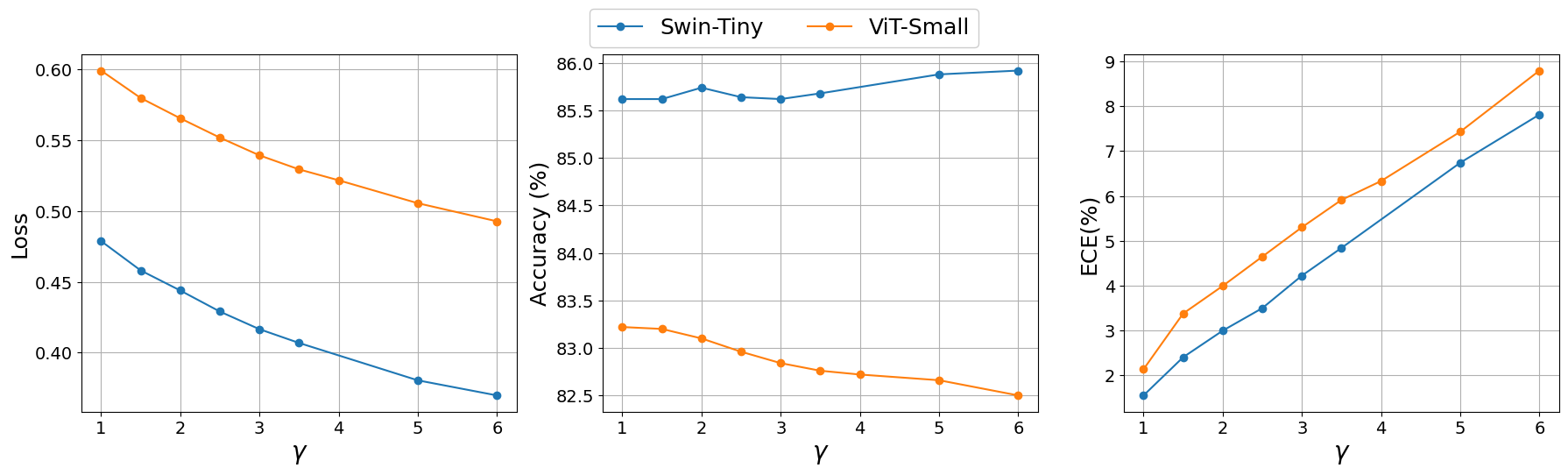}
        \subcaption{Dual focal}
    \end{minipage}
    \vspace{-2mm}
    \caption{\(\gamma\) selection for inverse focal loss using cross-validation on Tiny-Imagenet. } 
    \label{fig:gammaselection}
\end{figure}
\begin{table}[htbp]
\caption{Selected $\gamma$ values for Dual Focal and Inverse Focal loss across models and datasets.}
\vspace{-2mm}
\centering
\footnotesize
\setlength{\tabcolsep}{2.4pt}
\renewcommand{\arraystretch}{1.15}
\begin{tabular}{l|cccc|cccc|cc}
\toprule
\textbf{Loss} 
  & \multicolumn{4}{c|}{\textbf{CIFAR-10}} 
  & \multicolumn{4}{c|}{\textbf{CIFAR-100}} 
  & \multicolumn{2}{c}{\textbf{Tiny-ImageNet}} \\
\cmidrule(lr){2-5} \cmidrule(lr){6-9} \cmidrule(lr){10-11}
  & ResNet50 & ResNet110 & WRN & PreResNet56 
  & ResNet50 & ResNet110 & WRN & PreResNet56 
  & ViT-Small & Swin-Tiny \\
\midrule
Dual     & 5.0 & 4.5 & 2.6 & 4.0 
         & 5.0 & 6.1 & 3.9 & 4.5 
         & 1.0 & 1.0 \\
Inv. FL  & 2.0 & 3.0 & 1.5 & 2.0 
         & 1.0 & 1.5 & 1.0 & 1.0 
         & 1.0 & 1.0 \\
\bottomrule
\end{tabular}
\label{tb:gammaforinvdual}
\end{table}

\begin{figure*}[htbp]
\centering
\begin{minipage}{0.24\linewidth}
    \centering
    \includegraphics[width=1.0\linewidth]{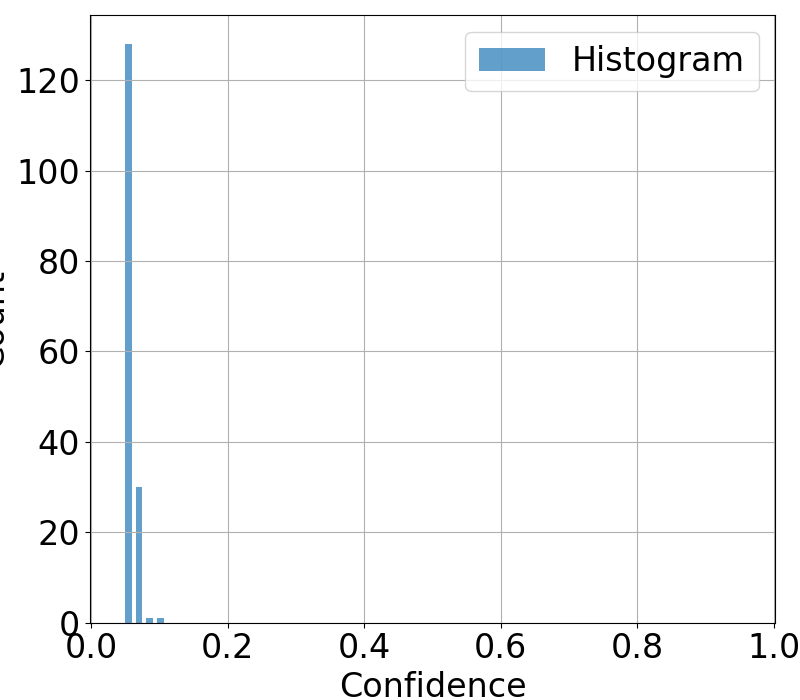}
    \subcaption{Histgram}
    \label{fig:hist}
\end{minipage}
\begin{minipage}{0.24\linewidth}
    \centering
    \includegraphics[width=1.0\linewidth]{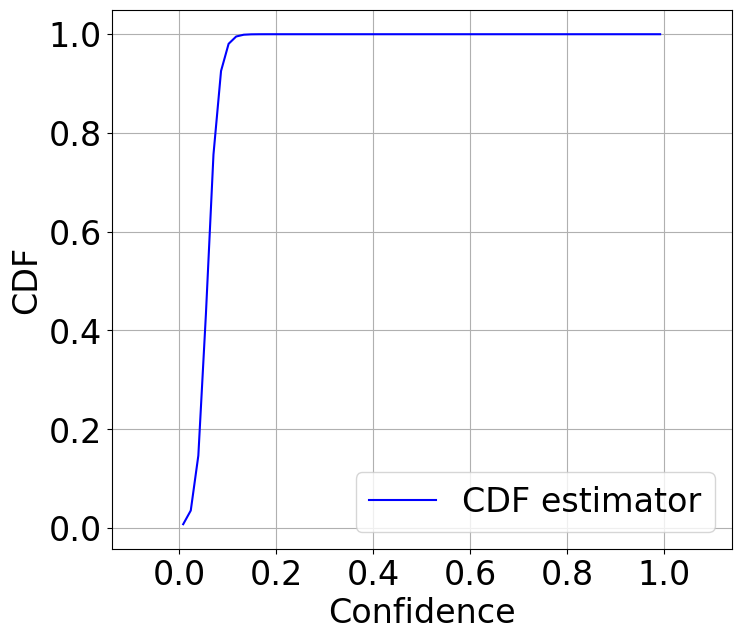}
    \subcaption{$\nu=0.01$}
    \label{fig:nu0.01}
\end{minipage}
\begin{minipage}{0.24\linewidth}
    \centering
    \includegraphics[width=1.0\linewidth]{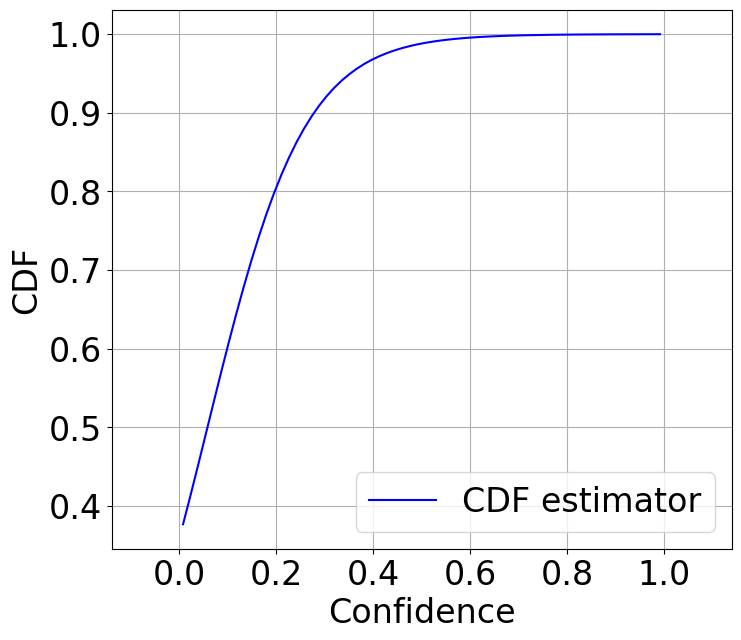}
    \subcaption{$\nu=0.1$}
    \label{fig:nu0.1}
\end{minipage}
\begin{minipage}{0.24\linewidth}
    \centering
    \includegraphics[width=1.0\linewidth]{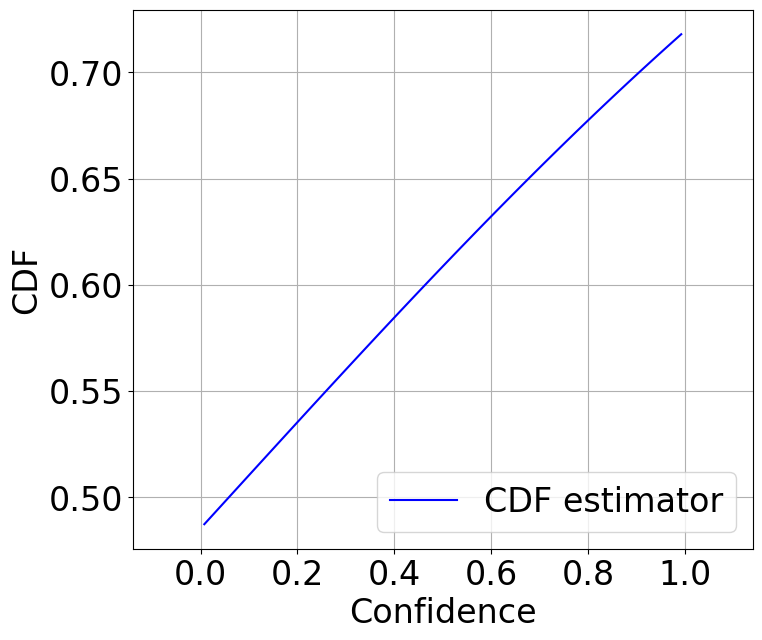}
    \subcaption{$\nu=1.0$}
    \label{fig:nu1}
\end{minipage}
\vspace{-2mm}
\caption{Effect of the smoothness parameter $\nu$ on the CDF approximation.}
\label{fig:smoothparameter}
\end{figure*}

\subsection{$\mathbf{\text{AU}_{\mathbb{P}_{\mathcal{X}_\tau}}}$ loss} \label{appsec:auloss}
This loss has two tunable parameters: the smooth parameter $\nu$ introduced by the CDF approximation and the threshold $\tau$ in Eq.~\eqref{eq:auloss}. \\
\noindent \textbf{Smooth parameter}.
The smoothness parameter $\nu > 0$ in the bin-based CDF approximation controls the steepness of the sigmoid function used to approximate the step function for each bin. As shown in Fig.~\ref{fig:smoothparameter}, smaller values of $\nu$ produce a sharper transition, closely mimicking the discrete step function and providing a more accurate approximation of the true CDF, but may lead to large gradients and unstable training. Larger values of $\nu$ result in a flatter, smoother transition, which stabilizes gradient-based optimization but may over-smooth the CDF and reduce ranking precision in low-confidence regions. In practice, $\nu$ is typically chosen as a small fraction of the bin width or tuned via cross-validation, balancing the trade-off between approximation fidelity and differentiability for efficient optimization. We set this parameter to 0.1 for all models on CIFAR-10 and 0.005 for ResNet models, 0.1 for WRN, and 0.01 for PreResNet56 on CIFAR-100. For transformer models, it is set to 0.05 for Swin-Tiny and ViT-Small.

\begin{figure}[htbp]
    \centering
    \begin{subfigure}{0.49\linewidth}
        \centering
        \includegraphics[width=2.1in]{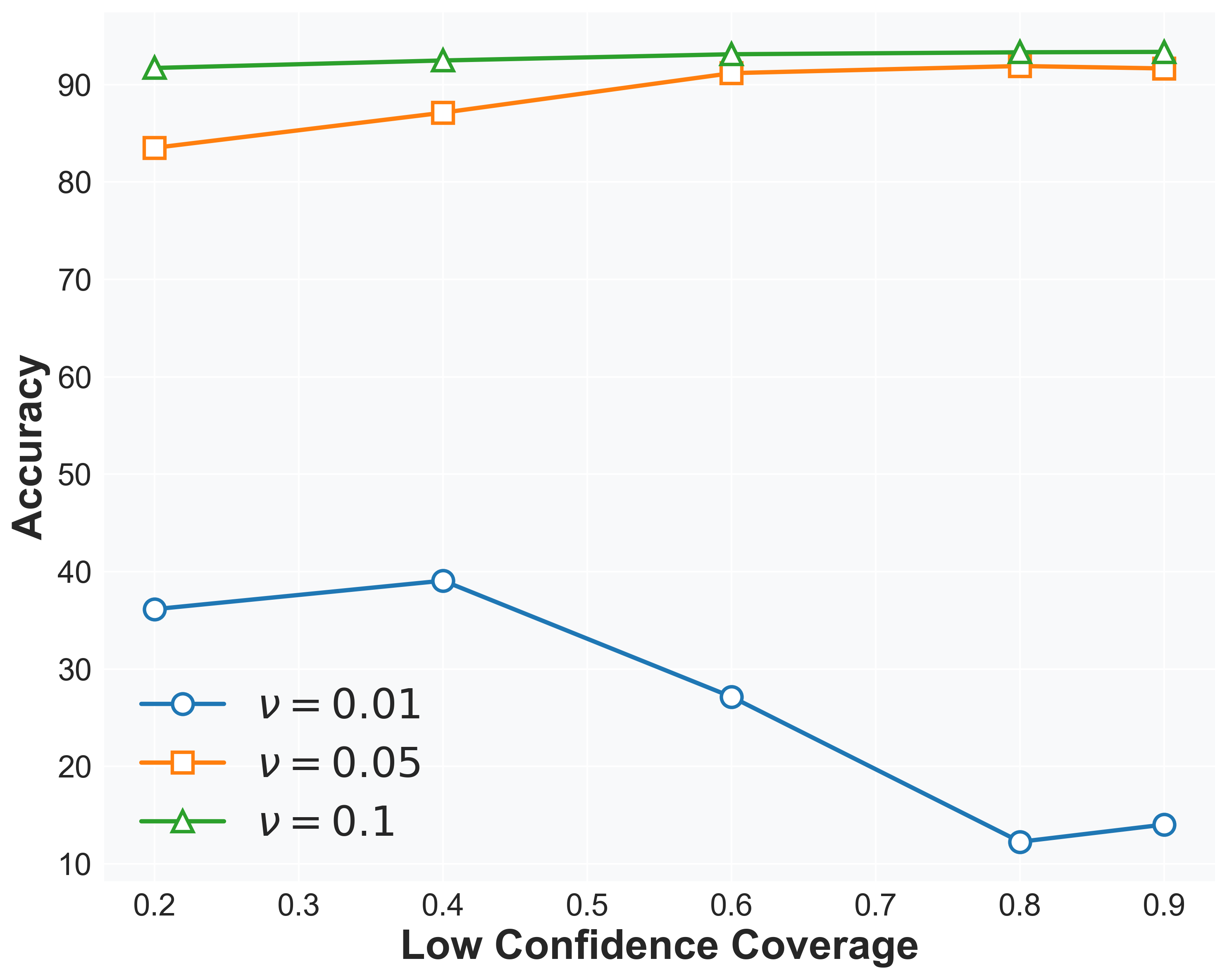}
        \caption{Accuracy}
    \end{subfigure}
    \hfill
    \begin{subfigure}{0.49\linewidth}
        \centering
        \includegraphics[width=2.1in]{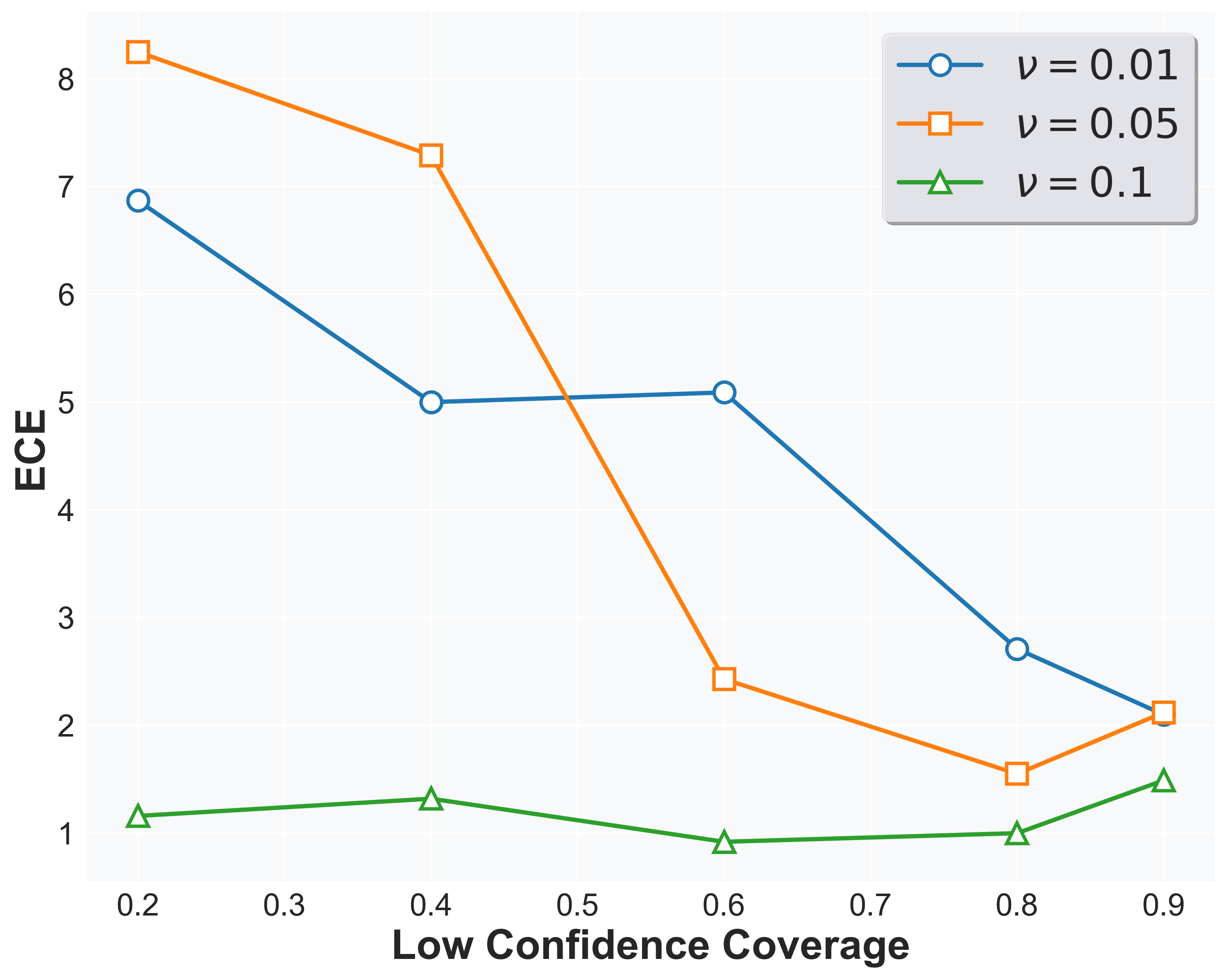}
        \caption{Gradient}
    \end{subfigure}
    \vspace{-2mm}
 \caption{Validation Performance of PreResNet56 on CIFAR-10 across low confidence coverage and $\nu$ Values}
    \label{fig:nuselection}
\end{figure}

\noindent\textbf{Coverage-Based Threshold $\tau$.} Instead of using a fixed threshold, we determine $\tau$ dynamically based on a desired coverage level, which specifies the fraction of samples considered low-confidence. Under this scheme, the threshold is set such that the bottom $\kappa$ fraction of samples ranked by the CSF $g(\boldsymbol{x})$ are treated as low-confidence, while the remaining samples are classified as high-confidence and assigned uniform weights. This approach allows the loss to adapt to the distribution of confidence scores within each batch or dataset, ensuring that a consistent proportion of uncertain samples is emphasized for calibration, regardless of 
confidence values. Coverage-based thresholding provides a more flexible and robust mechanism for targeting low-confidence regions compared to a fixed $\tau$, particularly in settings where confidence distributions vary across datasets or model architectures. The selected low-confidence coverage for different models and datasets is summarized in Table~\ref{tb:thresholdforau}. A subset of the selection results is presented in Table~\ref{fig:nuselection}.

\begin{table}[htbp]
\caption{Selected low confidence coverage for $\mathbf{\text{AU}_{P_{\mathcal{X}_\tau}}}$ loss across models and datasets.}
\vspace{-2mm}
\centering
% \footnotesize
\setlength{\tabcolsep}{2.4pt}
\renewcommand{\arraystretch}{1.15}
\begin{tabular}{cccc|cccc|cc}
\toprule
\multicolumn{4}{c|}{\textbf{CIFAR-10}} 
  & \multicolumn{4}{c|}{\textbf{CIFAR-100}} 
  & \multicolumn{2}{c}{\textbf{Tiny-ImageNet}} \\
\cmidrule(lr){1-4} \cmidrule(lr){5-8} \cmidrule(lr){9-10}
ResNet50 & ResNet110 & WRN & PreResNet56 
  & ResNet50 & ResNet110 & WRN & PreResNet56 
  & ViT-Small & Swin-Tiny \\
\midrule
0.88 & 0.75 & 0.8 & 0.75
  & 0.75 & 0.8 & 0.8 & 0.92
  & 0.4 & 0.6 \\
\bottomrule
\end{tabular}
\label{tb:thresholdforau}
\end{table}

\subsection{Additional Results} \label{appdix:Addresult}

% On CIFAR-100, both KDE-XE and Dual Focal loss show competitive cwECE results while maintaining relatively high accuracy. 

\noindent [\textbf{Top-label calibration}.] Table~\ref{tab:ece_a_ts_split} presents a comprehensive evaluation of top-label calibration, both before and after applying post-hoc calibration techniques, including TS, MetaMis, and MetaAcc, as measured by ECE$^{\text{a}}$. The results show that $\mathbf{\text{AU}_{\mathbb{P}_{\mathcal{X}_\tau}}}$ delivers strong top-label calibration performance over standard losses such as XE and FL-53, and remains competitive with Dual Focal. In contrast, FL-53 shows limited effectiveness in top-label calibration across CIFAR-10 and Tiny-ImageNet. 

\noindent [\textbf{Effect of post-hoc calibration method.}] Although post-hoc calibration techniques such as MetaMis and MetaAcc are designed to enhance calibration by leveraging confidence information similar to our proposed method, they exhibit a noticeable trade-off between calibration quality and predictive performance, as indicated in Table~\ref{tab:ece_a_ts_split}. While both methods occasionally reduce $\widehat{\text{ECE}}$, the improvements are often offset by substantial accuracy degradation. For instance, accuracy drops exceeding 10\% are observed in models like WRN on CIFAR-100. On Tiny-ImageNet, MetaAcc yields inconsistent results and, in some cases, negatively impacts calibration. These findings suggest that, despite their intended purpose, MetaMis and MetaAcc lack the robustness and reliability demonstrated by our proposed method. 

\begin{table*}[htbp]
\caption{Mean $\widehat{\text{ECE}}^a$ $\times 10^{2}$ ($\downarrow$) for original models, original models, and after applying post-hoc calibration methods including TS, MetaMis, and MetaAcc. Accuracy (\%) shown in brackets. For TS, temperature values are included in parentheses. For TS, the temperature $T$ is selected by minimizing ECE.}
\label{tab:ece_a_ts_split}
\centering
\renewcommand{\arraystretch}{1.15}
\setlength{\tabcolsep}{2pt}
\resizebox{1\textwidth}{!}{%
\begin{tabular}{@{}lllllllll@{}}
\toprule
                & \textbf{Model}       & \textbf{XE} & \textbf{FL-53} & \textbf{Inv. FL} & \textbf{Dual} & \textbf{KDE-XE} & \textbf{AURC} & \textbf{$\mathbf{\text{AU}_{\mathbb{P}_{\mathcal{X}_\tau}}}$} \\
\midrule
\multicolumn{9}{l}{\textbf{CIFAR10}} \\
\rowcolor{gray!20}  & ResNet50 & 2.97 [\textbf{94.5}] & 3.92 [93.6] & 4.03 [94.4] & \textbf{1.21} [94.1] & 3.07 [94.5] & 2.42 [93.3] & 1.35 [92.5] \\
  & \hspace{1em}+TS & 2.60 [\textbf{94.5}] (1.1) & 1.71 [93.6] (0.9) & 3.64 [94.4] (1.2) & \textbf{0.94} [94.1] (1.0) & 2.66 [94.5] (1.1) & 2.41 [93.3] (1.0) & 1.04 [92.5] (1.2) \\
  & \hspace{1em}+MetaMis & 0.83 [84.7] & 1.99 [86.6] & 0.85 [84.3] & 0.79 [85.5] & 0.68 [84.3] & 1.21 [87.1] & \textbf{0.67} [\textbf{88.8}] \\
  & \hspace{1em}+MetaAcc & \textbf{0.44} [77.5] & 1.85 [74.9] & 0.58 [78.7] & 1.04 [72.0] & 0.51 [76.9] & 1.77 [\textbf{90.0}] & 0.77 [89.7] \\
\rowcolor{gray!20}  & ResNet110 & 3.00 [94.5] & 4.39 [93.8] & 4.76 [93.9] & 1.93 [93.7] & 2.77 [\textbf{94.7}] & 2.36 [93.0] & \textbf{1.33} [92.4] \\
  & \hspace{1em}+TS & 2.63 [94.5] (1.1) & 1.95 [93.8] (0.8) & 4.26 [93.9] (1.3) & 1.19 [93.7] (0.9) & 2.45 [\textbf{94.7}] (1.1) & 2.50 [93.0] (1.0) & \textbf{1.00} [92.4] (1.1) \\
  & \hspace{1em}+MetaMis & 0.69 [84.5] & 2.33 [87.6] & 1.68 [86.7] & 1.22 [86.1] & 0.70 [86.0] & 1.48 [87.6] & \textbf{0.62} [\textbf{88.5}] \\
  & \hspace{1em}+MetaAcc & 0.44 [78.1] & 1.98 [74.2] & 2.42 [88.1] & 1.36 [77.2] & \textbf{0.44} [78.0] & 2.22 [\textbf{91.6}] & 0.75 [89.8] \\
\rowcolor{gray!20}  & WRN & 1.93 [95.9] & 5.72 [95.5] & 2.37 [\textbf{96.0}] & 3.93 [95.7] & 2.02 [95.8] & 1.72 [94.9] & \textbf{1.50} [95.3] \\
  & \hspace{1em}+TS & 1.31 [95.9] (1.2) & 2.90 [95.5] (0.9) & 1.86 [\textbf{96.0}] (1.2) & 3.14 [95.7] (1.0) & 1.39 [95.8] (1.2) & 1.20 [94.9] (1.1) & \textbf{0.41} [95.3] (1.3) \\
  & \hspace{1em}+MetaMis & 0.50 [82.6] & 2.62 [86.2] & \textbf{0.44} [82.6] & 2.26 [84.8] & 0.56 [82.5] & 0.74 [\textbf{87.3}] & 0.60 [84.6] \\
  & \hspace{1em}+MetaAcc & 0.67 [71.6] & 2.35 [69.2] & \textbf{0.52} [72.4] & 1.97 [66.4] & 0.80 [71.0] & 0.61 [\textbf{82.1}] & 0.81 [72.6] \\
\rowcolor{gray!20}  & PreResNet56 & 2.62 [94.0] & 5.42 [93.0] & 3.85 [\textbf{94.1}] & 2.05 [93.7] & 2.55 [94.1] & 2.56 [92.6] & \textbf{0.85} [92.8] \\
  & \hspace{1em}+TS & 2.52 [94.0] (1.0) & 0.97 [93.0] (0.7) & 3.75 [\textbf{94.1}] (1.1) & 0.99 [93.7] (0.9) & 2.44 [94.1] (1.0) & 2.76 [92.6] (1.0) & \textbf{0.50} [92.8] (1.1) \\
  & \hspace{1em}+MetaMis & 0.94 [85.9] & 1.53 [88.7] & 1.05 [84.8] & 1.08 [86.0] & 0.90 [85.3] & 1.49 [87.8] & \textbf{0.39} [\textbf{89.3}] \\
  & \hspace{1em}+MetaAcc & \textbf{0.38} [75.9] & 2.12 [75.8] & 1.27 [82.7] & 1.45 [68.2] & 0.89 [81.9] & 2.49 [\textbf{91.5}] & 0.52 [90.5] \\
\midrule
\multicolumn{9}{l}{\textbf{CIFAR100}} \\
\rowcolor{gray!20}  & ResNet50 & 7.15 [74.4] & 4.56 [\textbf{76.4}] & 10.70 [75.6] & 4.21 [75.5] & 9.52 [76.2] & 8.28 [74.5] & \textbf{2.96} [74.7]\\
  & \hspace{1em}+TS & 1.73 [74.4] (1.3) & 1.30 [\textbf{76.4}] (0.9) & 3.78 [75.6] (1.4) & \textbf{1.17} [75.5] (0.9) & 2.93 [76.2] (1.4) & 2.77 [74.5] (1.4) & 2.32 [74.7] (1.1)\\
  & \hspace{1em}+MetaMis & 6.72 [63.7] & 2.71 [59.1] & \textbf{2.03} [47.6] & 3.16 [60.7] & 2.21 [49.8] & 9.94 [68.7] & 8.16 [\textbf{69.5}] \\
  & \hspace{1em}+MetaAcc & 6.97 [61.0] & \textbf{0.92} [38.6] & 1.15 [37.6] & 1.06 [40.9] & 1.09 [38.6] & 13.49 [\textbf{73.5}] & 9.08 [71.6] \\
\rowcolor{gray!20}  & ResNet110 & 7.25 [75.3] & 4.54 [77.5] & 10.86 [77.4] & 4.58 [76.3] & 9.04 [\textbf{78.3}] & 7.94 [74.3] & \textbf{3.01} [74.5] \\
  & \hspace{1em}+TS & 1.85 [75.3] (1.3) & \textbf{0.95} [77.5] (0.9) & 4.33 [77.4] (1.4) & 0.98 [76.3] (0.9) & 2.58 [\textbf{78.3}] (1.4) & 1.96 [74.3] (1.3) & 2.34 [74.5] (1.1) \\
  & \hspace{1em}+MetaMis & 6.55 [63.3] & 2.35 [58.4] & \textbf{1.81} [48.3] & 3.39 [62.9] & 1.89 [50.7] & 10.23 [69.2] & 7.49 [\textbf{70.6}] \\
  & \hspace{1em}+MetaAcc & 6.79 [61.1] & \textbf{0.76} [38.8] & 1.11 [39.1] & 1.19 [43.0] & 1.06 [40.6] & 13.50 [\textbf{73.6}] & 8.43 [72.9] \\
\rowcolor{gray!20}  & WRN & 5.31 [78.7] & 6.20 [79.1] & 6.07 [79.5] & \textbf{3.57} [78.4] & 5.66 [79.6] & 4.66 [76.5] & 4.43 [\textbf{79.8}] \\
  & \hspace{1em}+TS & 3.24 [78.7] (1.2) & \textbf{1.19} [79.1] (0.8) & 4.17 [79.5] (1.3) & 1.32 [78.4] (0.9) & 3.63 [79.6] (1.2) & 2.70 [76.5] (1.2) & 4.08 [\textbf{79.8}] (1.1) \\
  & \hspace{1em}+MetaMis & 3.57 [54.4] & 2.08 [56.8] & \textbf{1.18} [46.4] & 2.52 [57.8] & 1.27 [46.9] & 7.22 [\textbf{69.5}] & 1.25 [46.2] \\
  & \hspace{1em}+MetaAcc & 3.73 [48.2] & \textbf{0.56} [34.3] & 0.60 [34.9] & 0.68 [38.0] & 0.70 [37.0] & 9.21 [\textbf{74.0}] & 0.63 [32.3] \\
\rowcolor{gray!20}  & PreResNet56 & 7.09 [74.3] & 2.49 [72.9] & 9.80 [74.2] & \textbf{1.00} [73.0] & 7.41 [\textbf{74.7}] & 4.76 [71.9] & 2.10 [73.9] \\
  & \hspace{1em}+TS & 1.33 [74.3] (1.4) & 1.38 [72.9] (0.9) & 2.48 [74.2] (1.4) & \textbf{0.96} [73.0] (1.0) & 1.61 [\textbf{74.7}] (1.4) & 2.14 [71.9] (1.2) & 1.67 [73.9] (1.1) \\
  & \hspace{1em}+MetaMis & 9.55 [69.3] & \textbf{5.00} [67.8] & 10.55 [69.2] & 7.66 [\textbf{70.4}] & 9.75 [70.0] & 9.66 [68.2] & 7.64 [70.3] \\
  & \hspace{1em}+MetaAcc & 10.51 [70.6] & \textbf{5.28} [68.6] & 12.51 [71.8] & 7.84 [70.9] & 10.53 [71.0] & 11.16 [71.2] & 8.74 [\textbf{73.2}] \\
\midrule
\multicolumn{9}{l}{\textbf{Tiny-ImageNet}} \\
\rowcolor{gray!20}  & ViT-Small & 4.55 [82.9] & 3.97 [81.4] & 6.22 [82.9] & 3.20 [82.1] & 3.01 [83.3] & \textbf{1.23} [82.7] & 1.36 [\textbf{83.6}] \\
  & \hspace{1em}+TS & 1.81 [82.9] (1.2) & 1.89 [81.4] (0.9) & 1.99 [82.9] (1.3) & 1.79 [82.1] (1.1) & 1.43 [83.3] (1.1) & \textbf{1.18} [82.7] (1.0) & 1.31 [\textbf{83.6}] (1.0) \\
  & \hspace{1em}+MetaMis & 5.18 [76.9] & \textbf{3.12} [76.7] & 4.65 [74.6] & 4.99 [75.5] & 4.74 [77.8] & 3.57 [78.1] & 4.02 [\textbf{78.9}] \\
  & \hspace{1em}+MetaAcc & 7.06 [80.5] & \textbf{1.42} [64.3] & 4.81 [70.7] & 4.13 [70.3] & 6.49 [81.3] & 4.10 [81.7] & 5.38 [\textbf{82.3}] \\
\rowcolor{gray!20}  & Swin-Tiny & 1.63 [85.4] & 7.69 [84.3] & 2.73 [\textbf{85.6}] & 1.52 [85.4] & 1.52 [85.4] & 1.88 [82.0] & \textbf{1.30} [84.1] \\
  & \hspace{1em}+TS & \textbf{1.12} [85.4] (1.0) & 1.72 [84.3] (0.8) & 1.40 [\textbf{85.6}] (1.1) & 1.31 [85.4] (1.0) & 1.34 [85.4] (1.1) & 1.46 [82.0] (0.9) & 1.12 [84.1] (1.0) \\
  & \hspace{1em}+MetaMis & 3.29 [81.7] & \textbf{1.74} [80.7] & 2.85 [81.2] & 3.25 [\textbf{81.8}] & 3.15 [81.3] & 1.95 [77.8] & 2.21 [79.7] \\
  & \hspace{1em}+MetaAcc & 3.99 [84.2] & 2.29 [82.5] & 3.67 [\textbf{84.3}] & 4.16 [84.2] & 3.98 [84.2] & \textbf{1.93} [79.7] & 2.91 [83.0] \\
\bottomrule
\end{tabular}}
\end{table*}

\end{document}